\documentclass{article} % For LaTeX2e
\PassOptionsToPackage{table}{xcolor}
\usepackage{iclr2027_conference,times}

\usepackage{amsmath,amsfonts,bm}

\def\eqref#1{equation~\ref{#1}}
\def\1{\bm{1}}

\DeclareMathAlphabet{\mathsfit}{\encodingdefault}{\sfdefault}{m}{sl}
\SetMathAlphabet{\mathsfit}{bold}{\encodingdefault}{\sfdefault}{bx}{n}

\usepackage{hyperref}
\usepackage{url}
\usepackage{booktabs}
\usepackage{graphicx}
\usepackage{amsthm}
\usepackage{algorithm}
\usepackage{algpseudocode}

\newtheorem{proposition}{Proposition}
\newtheorem{lemma}{Lemma}
\newtheorem{theorem}{Theorem}

\theoremstyle{definition}

\usepackage{caption}
\usepackage{multirow}
\usepackage[inline]{enumitem}

\usepackage{booktabs}
\usepackage[table]{xcolor}
\usepackage{graphicx}

\newcommand{\sold}{\mathrm{old}}
\newcommand{\sref}{\mathrm{ref}}
\newcommand{\sg}{\mathrm{sg}}

\newcommand{\piold}{\pi_{\sold}}

\title{MeanFlowAdvantage: Stable Reward Fine-Tuning for Few-Step Average-Velocity Generators}

\author{Haocheng Tang\thanks{Equal contribution.}\hspace{0.45em}\thanks{Corresponding author.} \\
Northeastern University\\
Boston, MA 02115, USA \\
\texttt{tang.haoc@northeastern.edu}
\And
Tianchi Xie\footnotemark[1] \\
Tsinghua University\\
Beijing, 100084, PRC \\
\texttt{xietc24@mails.tsinghua.edu.cn}
\And
Xingqiao Lin \\
Carnegie Mellon University \\
Pittsburgh, PA 15213, USA \\
\texttt{xingqiao@andrew.cmu.edu}
}

\iclrfinalcopy % Uncomment for camera-ready version, but NOT for submission.
\begin{document}

\maketitle

\begin{abstract}
MeanFlow enables efficient few-step generation by predicting interval-average velocities, but this representation creates a mismatch for reward fine-tuning: existing advantage-based objectives are typically defined on instantaneous velocities or equivalent $x_0$-space predictions, whereas inference directly uses the learned average-velocity map. We introduce \textbf{MeanFlowAdvantage}, a signed advantage-weighted least-squares objective for average-velocity generators. Our key construction uses a shared, detached MeanFlow derivative correction to express the reward objective in prediction space while making rollout and reference regularization exact penalties on the average-velocity network deployed at inference. The resulting formulation preserves MeanFlow's native few-step sampler and provides a direct mechanism for transferring reward improvements to the deployed flow map. On SD3.5-Medium, MeanFlowAdvantage improves all eight reported metrics over the matched four-step MeanFlowNFT baseline and, with only four NFEs, matches or exceeds the 40-step DiffusionNFT baseline on six of eight metrics. The same objective also transfers to DNA promoter design, where it supports both teacher-free on-policy RL for a generator defined on a manifold and teacher-guided reward-graded distillation, with the latter yielding the lowest one-step Sei profile MSE among the compared configurations.
\end{abstract}

\section{Introduction}
\label{sec:intro}

Few-step flow matching changes the representation that post-training must optimize. In standard flow matching, the model predicts an instantaneous velocity whose numerical integration produces samples \citep{lipman2023flow,liu2023rectified,albergo2023stochastic}. MeanFlow \citep{geng2025meanflow,gu2026anyflow}, in contrast, predicts an \emph{interval-average} velocity whose learned flow map can traverse a finite time interval in a single network evaluation. This representation is precisely what enables one-step and few-step sampling. However, it also creates a mismatch for reward fine-tuning: existing advantage-based objectives are typically formulated for instantaneous velocities or equivalent prediction spaces, whereas MeanFlow inference directly executes the learned average-velocity map. A useful reward objective should therefore improve sample quality without losing alignment with the representation that makes few-step generation possible.

Existing reinforcement learning (RL) approaches do not directly resolve this mismatch. Reverse-process methods such as DDPO, DPOK, and Flow-GRPO optimize stochastic denoising trajectories \citep{black2024ddpo,fan2023dpok,liu2025flowgrpo}, while forward-process methods incorporate reward signals into prediction objectives. DiffusionNFT uses implicit positive and negative policies \citep{zheng2026diffusionnft}, Advantage Weighted Matching reweights a matching surrogate \citep{xue2025awm}, and AdvantageFlow applies signed advantage-weighted least squares with rollout regularization \citep{kveton2026advantageflow}. These formulations naturally operate on instantaneous representations. Simply replacing such a predictor with a MeanFlow output is not equivalent: an interval-average velocity represents the transport over an entire interval rather than the instantaneous direction at one time point. Moreover, an induced predictor alone does not ensure that rollout or reference regularization constrains the positive-length average-velocity maps used by the sampler.

Recent forward-process work has shown that MeanFlow can support reward fine-tuning through an induced instantaneous predictor \citep{huang2026meanflownft}. This raises two questions: how can signed advantages be used to stably improve the finite-interval MeanFlow map, and can the same reward-based formulation also distill improved behavior from multi-step teachers into few-step MeanFlow models?

We introduce \textbf{MeanFlowAdvantage} (MFA), a signed advantage-weighted least-squares objective that provides a unified solution to both problems. MFA applies reward supervision in prediction space while regularizing the underlying average-velocity model against both the rollout policy and a frozen reference. A shared, stop-gradient correction makes these regularization terms reduce to direct differences between the corresponding MeanFlow velocity predictions, avoiding interference between reward fitting and model anchoring. This same formulation naturally extends from on-policy reward fine-tuning to reward-guided distillation, allowing few-step MeanFlow students to learn from improved multi-step teacher trajectories. Our contributions are threefold:
\begin{itemize}
\item \textbf{Reward fine-tuning aligned with average-velocity generation.} We formulate signed advantage optimization directly for finite-interval MeanFlow maps, while preserving the induced prediction space needed for reward supervision. This yields a representation-aligned objective that regularizes the average-velocity predictions used for generation rather than only the induced instantaneous boundary predictor.

\item \textbf{State-of-the-art few-step performance in our evaluation.} On SD3.5-Medium, MFA achieves the strongest four-NFE results among the compared MeanFlow reward-finetuning methods, improving all eight reported metrics over MeanFlowNFT. With only four NFEs, it also matches or exceeds the 40-NFE DiffusionNFT baseline on six of eight metrics.

\item \textbf{One objective for both on-policy RL and distillation.} On DNA promoters (Appendix~\ref{sec:expts-dna}), the same loss runs as teacher-free RL on a manifold generator and, with a reward-guided teacher in place of the rollout, transfers graded teacher improvements into a one-step student, achieving the best Sei profile error among the compared configurations.
\end{itemize}

\section{Preliminaries}
\label{sec:prelim}

We use the image convention $t=0$ for data and $t=1$ for noise. The earlier interval endpoint is $s\le t$, and $c$ denotes conditioning. We omit $c$ from the notation when clear from context, and distinguish average velocity $u$, instantaneous velocity $v$, and a training-time induced predictor $V$. The promoter implementation uses the opposite time convention, stated separately in Section~\ref{sec:expts-dna}.

\paragraph{Flow matching}
\label{sec:fm}

Flow matching \citep{lipman2023flow} learns a continuous-time transport from a simple noise distribution to the conditional data distribution $q(x_0\mid c)$. It represents this transport through a time-dependent velocity field whose probability-flow ODE is integrated from noise at $t=1$ to data at $t=0$ at inference time. Rectified flow specializes this construction to a linear interpolation between data and noise. For $x_0\sim q(\cdot\mid c)$ and independent $\epsilon\sim\mathcal N(0,I)$, it defines
\begin{equation}
x_t=(1-t)x_0+t\epsilon,\qquad v_t=\epsilon-x_0,\qquad \mathcal L_{\mathrm{FM}}=\mathbb E\|v_\theta(x_t,t)-v_t\|_2^2 .
\label{eq:xt}
\end{equation}
The population minimizer of this squared loss is the marginal instantaneous velocity $v(x_t,t,c)=\mathbb E[v_t\mid x_t,t,c]$. For rectified flow \citep{liu2023rectified}, the same instantaneous-velocity prediction can equivalently be expressed in $x_0$-space as $f_\theta=x_t-t v_\theta$, while $x_0=x_t-t v_t$. Therefore, $f_\theta-x_0=-t(v_\theta-v_t)$ and $\|f_\theta-x_0\|_2^2=t^2\|v_\theta-v_t\|_2^2$.
This prediction-space representation provides a convenient form for incorporating reward signals into forward-process objectives.

\paragraph{AdvantageFlow}
\label{sec:af}

Building on this representation, AdvantageFlow~\citep{kveton2026advantageflow} incorporates signed reward information through a quadratic objective on the $x_0$-space predictions of the learner, rollout model, and frozen reference:
\begin{equation}
\ell^{\mathrm{AF}}=A\|f_\theta-x_0\|_2^2+\gamma\|f_\theta-f_{\mathrm{old}}\|_2^2+\lambda\|f_\theta-f_{\mathrm{ref}}\|_2^2.
\label{eq:af-loss}
\end{equation}
Here $A$ is a signed advantage and $\gamma,\lambda\ge0$. The loss is strictly convex in the prediction when $A+\gamma+\lambda>0$, even when $A<0$. For standard flow models, this objective acts naturally on a prediction obtained from the instantaneous velocity. MeanFlow, however, deploys an interval-average velocity, so reward fine-tuning must additionally account for the finite-interval map used by its sampler.

\paragraph{MeanFlow and reward fine-tuning}
\label{sec:meanflow}

Along a trajectory $\dot x_\tau=v(x_\tau,\tau)$, MeanFlow represents transport over a finite interval $[s,t]$ through the average velocity
\begin{equation}
u(x_t,s,t)=\frac{1}{t-s}\int_s^t v(x_\tau,\tau)\,\mathrm d\tau,\qquad \lim_{s\to t}u(x_t,s,t)=v(x_t,t).
\label{eq:u-avg}
\end{equation}
Holding $s$ fixed and differentiating along the trajectory yields the MeanFlow identity
\begin{equation}
u+(t-s)\bigl[\partial_tu+(\partial_xu)v\bigr]=v.
\label{eq:mf-identity}
\end{equation}
For an exact average velocity, $x_s=x_t-(t-s)u(x_t,s,t)$ is the exact flow map; replacing $u$ by $u_\theta$ gives the learned sampler $x_s=x_t-(t-s)u_\theta(x_t,s,t)$. Thus, unlike standard flow matching, MeanFlow inference directly uses the  average-velocity maps predicted by $u_\theta$. Reward fine-tuning should therefore remain aligned with these deployed maps. MeanFlowNFT uses the MeanFlow identity to construct an instantaneous predictor $V$ from $u$ and applies an NFT objective to $V$~\citep{huang2026meanflownft}. We use the same induction with the signed quadratic in~\eqref{eq:af-loss}, while regularizing the corresponding average-velocity maps directly; Section~\ref{sec:nft-link} makes the connection to NFT explicit.

\section{MeanFlowAdvantage}
\label{sec:method}

The preliminaries establish AdvantageFlow in prediction space and MeanFlow in average-velocity space. MeanFlowAdvantage connects these representations while keeping optimization aligned with the finite-interval maps deployed by the sampler. We first derive a shared MeanFlow correction that links the average velocity $u$ to the prediction-space objective, and show that the resulting quadratic is exactly an advantage-weighted MeanFlow regression with rollout and reference anchors on the deployed maps. We then specify the samples, advantages, and training intervals used for optimization, followed by an analysis of the resulting objective and its fixed point.

\subsection{Linking average velocity to prediction space}
\label{sec:induce}\label{sec:F}\label{sec:cd}

The naive substitution $x_t-t\,u_\theta(x_t,s,t)$ does not give the corresponding $x_0$-space prediction when $s<t$. In rectified flow, $x_t-tv$ depends on the instantaneous velocity at time $t$, whereas $u$ averages velocity over the interval $[s,t]$. The MeanFlow identity~\eqref{eq:mf-identity} relates the two as $v=u+(t-s)\,\mathrm{D}u$, where $\mathrm{D}u=\partial_tu+(\partial_xu)v$ denotes the rate of change of $u$ along the trajectory. This correction therefore links the average-velocity representation to the prediction space used by AdvantageFlow. Given a rollout sample $x_0$, we draw noise $\epsilon$ and an interval $s\le t$, form $x_t$ by~\eqref{eq:xt}, and set $v_t=\epsilon-x_0$. With finite-difference step size $\delta>0$, we estimate the correction once from the rollout network along the sampled direction $v_t$:
\begin{equation}
d=
\frac{
u_{\sold}(x_t+\delta v_t,s,t+\delta)
-
u_{\sold}(x_t-\delta v_t,s,t-\delta)
}{2\delta}
=
\partial_tu_{\sold}
+
(\partial_xu_{\sold})\,v_t
+
O(\delta^2).
\label{eq:cd}
\end{equation}
A Jacobian-vector product computes the same quantity without the $O(\delta^2)$ error. Appendix~\ref{app:impl} gives the finite-difference implementation and boundary handling. We detach $d$, with $\sg[\cdot]$ denoting stop-gradient, and share the same value across the learner, rollout model, and reference model, indexed by $k$:
\begin{equation}
V_k=u_k(x_t,s,t)+(t-s)\,\sg[d],\qquad
F_k=x_t-t\,V_k,\qquad
k\in\{\theta,\sold,\sref\}.
\label{eq:V}
\end{equation}
Here $V_k$ is the corrected instantaneous-velocity representation and $F_k$ its corresponding $x_0$-space prediction. As the same correction is used for all three networks, it cancels in every pairwise difference:
\begin{equation}
F_\theta-F_k=-t\,(u_\theta-u_k),\qquad
V_\theta-V_k=u_\theta-u_k,\qquad
k\in\{\sold,\sref\}.
\label{eq:shared-cd}
\end{equation}
The rollout and reference anchors therefore compare the average-velocity networks directly on the sampled interval $(s,t)$. With a separate derivative for each network, each anchor residual would instead contain an additional term $t(t-s)(d_\theta-d_k)$ depending on the sampled direction $v_t$; Section~\ref{sec:analysis} shows why this matters.

\subsection{Advantage-weighted least squares on the induced prediction}
\label{sec:obj}

With the prediction~\eqref{eq:V}, the MFA loss is the AdvantageFlow quadratic~\eqref{eq:af-loss}:
\begin{equation}
 \ell_\theta=A\|F_\theta-x_0\|_2^2+\gamma\|F_\theta-F_{\sold}\|_2^2+\lambda\|F_\theta-F_{\sref}\|_2^2,
 \qquad \mathcal L_{\mathrm{MFA}}(\theta)=\mathbb E_{x_0\sim\piold}[\ell_\theta].
 \label{eq:mfa-loss}
\end{equation}
Within one update, the samples, the advantages, and the two anchor predictions are held fixed.

\textbf{MFA is an advantage-weighted MeanFlow regression.} Denote the regression target of MeanFlow training~\citep{geng2025meanflow} by
\begin{equation}
 \bar u=v_t-(t-s)\,d .
 \label{eq:ubar}
\end{equation}
Since $x_t-x_0=t\,v_t$, we have $F_\theta-x_0=-t(u_\theta-\bar u)$. Combined with~\eqref{eq:shared-cd}, this turns the loss into an exact identity in average-velocity space:
\begin{equation}
 \ell_\theta=t^2\Bigl(A\,\|u_\theta-\bar u\|_2^2+\gamma\,\|u_\theta-u_{\sold}\|_2^2+\lambda\,\|u_\theta-u_{\sref}\|_2^2\Bigr).
 \label{eq:mfa-u}
\end{equation}
The data term is the standard MeanFlow loss weighted by the advantage, and the anchors keep the deployed map close to the rollout and reference maps on the sampled interval. The prediction $F$ is only a device to derive~\eqref{eq:mfa-u}; what is trained and regularized is $u_\theta$ itself. When $s=t$, the correction vanishes and~\eqref{eq:mfa-u} reduces to AdvantageFlow on the boundary velocity $u_\theta(x_t,t,t)$ (Appendix~\ref{app:reparam}). The positive-length intervals are therefore what is new for few-step generators.

\textbf{Per-sample solution and signed advantages.} Treat $u=u_\theta(x_t,s,t)$ as a free vector. The loss~\eqref{eq:mfa-u} has curvature $2t^2(A+\gamma+\lambda)$, so whenever $A+\gamma+\lambda>0$ it is strictly convex and has the unique minimizer
\begin{equation}
 u^\star=\frac{A\,\bar u+\gamma\,u_{\sold}+\lambda\,u_{\sref}}{A+\gamma+\lambda}.
 \label{eq:F-star}
\end{equation}
The three weights sum to one. For $A>0$, $u^\star$ moves toward the target $\bar u$. For $A<0$, the weight on $\bar u$ is negative, so $u^\star$ is pushed away from the target, past the rollout output. A poor sample therefore produces a push away from itself instead of being ignored, and the loss remains convex. We use $\gamma=1.1$, $\lambda=10^{-3}$, and $A\in[-1,1]$, so $A+\gamma+\lambda\ge0.101$ always holds. Section~\ref{sec:analysis} shows that $\gamma$ also controls how far each update can move.

\textbf{Adaptive scale.} As in the DiffusionNFT trainer, the image implementation divides each sample's loss by a detached positive scale:
\begin{equation}
 \mathcal L_{\mathrm{train}}=\mathbb E[\ell_\theta/w],\qquad w=\max\bigl(\sg[\operatorname{mean}|F_\theta-x_0|],10^{-5}\bigr),
 \label{eq:scale}
\end{equation}
where the mean is over latent coordinates. The three terms of a sample share the same $w$, so $w$ does not change the per-sample minimizer~\eqref{eq:F-star}. It only changes how much each sample contributes to the gradient. The ablations show that this matters for stability (Section~\ref{sec:ablation}). The analysis below is stated for the unscaled loss.

\subsection{Samples, advantages, and intervals}
\label{sec:adv}\label{sec:tr}\label{sec:rollout}

We keep the outer loop of MeanFlowNFT~\citep{huang2026meanflownft}, so any difference from it comes from the objective alone.

\textbf{Rollouts.} Starting from Gaussian noise, the frozen rollout model runs the $N$-step sampler on a decreasing grid $1=t_N>\cdots>t_0=0$:
\begin{equation}
 x_{t_{i-1}}=x_{t_i}-(t_i-t_{i-1})\,u_{\sold}(x_{t_i},t_{i-1},t_i),\qquad i=N,\ldots,1,
 \label{eq:nstep}
\end{equation}
with $N=4$ and no CFG. The rollout parameters are an EMA of the learner and are refreshed between updates. The reference is the frozen AnyFlow initialization~\citep{gu2026anyflow}. Fine-tuning uses LoRA~\citep{hu2022lora}.

\textbf{Advantages.} For $L$ prompts with $K$ images each, rewards are centered within each prompt and divided by a global standard deviation:
\begin{equation}
 A^{i,k}=\operatorname{clip}\Bigl(\frac{R^{i,k}-\widehat R^{i}}{\max(Z,\varepsilon_Z)},-1,1\Bigr),\qquad
 \widehat R^{i}=\frac1K\sum_{k}R^{i,k},\qquad Z^2=\frac{1}{LK}\sum_{i,k}\bigl(R^{i,k}-\widehat R^{i}\bigr)^2.
 \label{eq:adv}
\end{equation}
With several rewards, each is standardized separately and the results are combined with fixed weights before clipping. The advantages are signed: an image worse than its prompt average receives $A<0$.

\textbf{Intervals.} A few-step sampler takes jumps of positive length, so training must include such intervals. Each sample receives one of three interval types:
\begin{equation}
 s=t\ \text{(boundary)},\qquad s=0\ \text{(full jump to data)},\qquad 0<s<t\ \text{(general interval)},
 \label{eq:three-mode}
\end{equation}
with probabilities $(0.5,0.25,0.25)$ and the time shift of the base scheduler. Boundary samples keep the instantaneous velocity accurate. The other two types train the maps that the sampler actually runs. Algorithm~\ref{alg:mfa} in Appendix~\ref{app:impl} summarizes one update.

\section{What does MeanFlowAdvantage learn?}
\label{sec:analysis}

The identity~\eqref{eq:mfa-u} holds exactly for every sample. We now take expectations and answer three questions. Which function does the loss drive the network toward? How large is each update, differing from MeanFlowNFT? When does the few-step sampler inherit the improvement? The analysis is at the population level, with a free function in place of the network. We write $Z=(x_t,s,t,c)$ for the network input and treat the advantage as $A=A(x_0,c)$. Proofs are in Appendix~\ref{app:proofs}.

\paragraph{Assumptions.} We use three assumptions: \emph{(A1) Exact correction:} $d=\partial_tu_{\sold}+(\partial_xu_{\sold})v_t$, which a JVP computes exactly while~\eqref{eq:cd} has error up to $O(\delta^2)$ ; \emph{(A2) Positive weights:} $A+\gamma>0$ for every sample; and \emph{(A3) Calibrated rollout:} $\mathbb E_{\piold}[\bar u\mid Z]=u_{\sold}(Z)$. Assumption (A3) means that the rollout network is a MeanFlow fixed point on its own samples. Under (A1), this corresponds to the MeanFlow identity~\eqref{eq:mf-identity} for the velocity field of $\piold$. It is exact only for a perfectly calibrated rollout; we discuss this gap at the end of the section.

For each prompt $c$, define the reweighted distribution
\begin{equation}
 q(x_0\mid c)=\frac{\gamma+A(x_0,c)}{\gamma+\bar A_c}\;\piold(x_0\mid c),\qquad \bar A_c=\mathbb E_{\piold}[A\mid c].
 \label{eq:tilt}
\end{equation}
$q$ upweights positive-advantage samples and downweights negative ones; smaller $\gamma$ strengthens the reweighting. Let $v_q(x_t,t,c)=\mathbb E_q[\epsilon-x_0\mid x_t,t,c]$ be the rectified-flow velocity that transports noise to $q$.

\begin{theorem}[MFA fits the MeanFlow target of a reward-reweighted distribution]
\label{thm:tilt}
Assume (A1)--(A3). There is a constant $C$ that does not depend on $\theta$ such that
\begin{equation}
 \mathcal L_{\mathrm{MFA}}(\theta)\big|_{\lambda=0}
 =\mathbb E_c\Bigl[(\gamma+\bar A_c)\,\mathbb E_{q}\bigl[t^2\|u_\theta-\bar u\|_2^2\,\big|\,c\bigr]\Bigr]+C .
 \label{eq:thm-tilt}
\end{equation}
For any $\lambda\ge0$, the loss has a unique minimizer over free functions $u(Z)$:
\begin{equation}
 u^{+}=\alpha\,u_q+(1-\alpha)\,u_{\sref},\qquad
 \alpha=\frac{\gamma+\mathbb E_{\piold}[A\mid Z]}{\gamma+\lambda+\mathbb E_{\piold}[A\mid Z]},
 \label{eq:thm-tilt-min}
\end{equation}
where $u_q=v_q-(t-s)\bigl[\partial_tu_{\sold}+(\partial_xu_{\sold})\,v_q\bigr]$.
\end{theorem}

For $\lambda=0$, MFA on rollout samples is equivalent to ordinary MeanFlow training under $q$, without sampling from $q$: $u_q$ is the corresponding MeanFlow target with the rollout derivative stop-gradiented. Positive $\lambda$ blends this target with the reference map. The shared correction is essential: because $u_\theta-u_{\sold}$ depends only on $Z$~\eqref{eq:shared-cd}, the rollout anchor becomes a MeanFlow regression up to a constant; separate derivatives introduce dependence on $v_t$ and break this equivalence.

\begin{proposition}[Update direction and step size]
\label{prop:step}
Assume (A1)--(A3), and let $\sigma_A^2(Z)=\operatorname{Var}_{\piold}(A\mid Z)$ and $\sigma_{\bar u}^2(Z)=\operatorname{tr}\operatorname{Cov}_{\piold}(\bar u\mid Z)$. Then
\begin{equation}
\begin{aligned}
u^{+}-u_{\sold}
&=
\frac{\operatorname{Cov}_{\piold}(A,\bar u\mid Z)+\lambda(u_{\sref}-u_{\sold})}
{\gamma+\lambda+\mathbb E_{\piold}[A\mid Z]},
\\[-1pt]
\|u^{+}-u_{\sold}\|_2
&\le
\frac{\sigma_A\sigma_{\bar u}+\lambda\|u_{\sref}-u_{\sold}\|_2}
{\gamma+\lambda+\mathbb E_{\piold}[A\mid Z]}.
\end{aligned}
\label{eq:step}
\end{equation}
\end{proposition}

The update direction is the covariance between the advantage and the MeanFlow target. Among all images that are consistent with the noisy input $Z$, the map moves toward the targets of high-advantage images and away from those of low-advantage images. The step size scales as $1/\gamma$, so $\gamma$ acts as an inverse step size, or trust-region radius, on the deployed map. The step is also bounded by the spread of advantages at $Z$: if all images consistent with $Z$ are equally good, nothing moves. A constant advantage $A\equiv1$ has zero covariance, so for $\lambda=0$ it gives $u^+=u_{\sold}$. It carries no reward information, and the loss only refits the rollout samples.

\begin{proposition}[Relation to MeanFlowNFT]
\label{prop:nft}
Apply the NFT loss~\citep{zheng2026diffusionnft} to the induced velocity~\eqref{eq:V} as in MeanFlowNFT, with parameter $\beta>0$, optimality probability $r\in[0,1]$, no reference term, and no adaptive weights. Under (A1) and (A3), its minimizer over free functions is
\begin{equation}
 u^{+}_{\mathrm{NFT}}=u_{\sold}+\frac{1}{\beta}\operatorname{Cov}_{\piold}(A,\bar u\mid Z),\qquad A=2r-1.
 \label{eq:nft-min}
\end{equation}
\end{proposition}

Compared with~\eqref{eq:step} at $\lambda=0$, both objectives move in the same direction, but NFT uses the fixed gain $1/\beta$ while MFA uses $1/(\gamma+\mathbb E_{\piold}[A\mid Z])$. By Theorem~\ref{thm:tilt}, the MFA target at every input is the MeanFlow target of one fixed distribution $q$; the NFT target has this form only if $\mathbb E_{\piold}[A\mid Z]$ does not depend on $Z$. For example, with $A=2r-1$ and $\gamma=1$, $q\propto r\,\piold$ is exactly the positive policy that motivates DiffusionNFT. MFA regresses onto its MeanFlow target $u_{\pi^+}$ everywhere, whereas NFT regresses onto $u_{\sold}+(2\rho/\beta)(u_{\pi^+}-u_{\sold})$ with $\rho(Z)=\mathbb E_{\piold}[r\mid Z]$. That target falls short of $u_{\pi^+}$ where $2\rho<\beta$ and overshoots it where $2\rho>\beta$ (Appendix~\ref{app:nft-proof}). This does not make MFA uniformly better, but its target has a simple distributional meaning at every noise level.

\begin{theorem}[When the sampler inherits the improvement]
\label{thm:transfer}
Let $\lambda=0$, and let $u^\ast$ be a fixed point of the target in Theorem~\ref{thm:tilt} when its own derivative is used:
\begin{equation}
 u^\ast=v_q-(t-s)\bigl[\partial_tu^\ast+(\partial_xu^\ast)\,v_q\bigr]\quad\text{for all }0\le s<t\le1 .
 \label{eq:fixed-point}
\end{equation}
Under the regularity conditions of Lemma~\ref{lem:consistency}, $u^\ast$ is the exact average velocity of the flow of $v_q$. For every number of steps $N$ and every grid, the sampler~\eqref{eq:nstep} with $u^\ast$ therefore draws exact samples from $q$. For each prompt,
\begin{equation}
 \mathbb E_{q}[R\mid c]-\mathbb E_{\piold}[R\mid c]=\frac{\operatorname{Cov}_{\piold}(A,R\mid c)}{\gamma+\bar A_c},
 \label{eq:J-improve}
\end{equation}
which is nonnegative whenever $A$ is a nondecreasing function of a scalar reward $R$.
\end{theorem}

Condition~\eqref{eq:fixed-point} is the same stop-gradient fixed point that MeanFlow pretraining aims for: the target uses the current network's derivative, and a network that matches its own target is an exact flow map. Theorem~\ref{thm:transfer} has two practical consequences. First, the condition must hold for \emph{every} interval. Boundary training alone ($s=t$) only pins down $v_q$ and leaves the maps used by the sampler unconstrained, which is why~\eqref{eq:three-mode} includes positive-length intervals. Second, the target distribution $q$ does not depend on $N$, so at the fixed point one update improves every sampling budget at once. Together with Proposition~\ref{prop:step}, it also shows the role of $\gamma$: a smaller $\gamma$ gives a larger possible improvement~\eqref{eq:J-improve}, but also larger and less stable steps~\eqref{eq:step}.

\textbf{What the analysis does not cover.} The theory assumes an exact correction (A1), a calibrated rollout (A3), and free functions, whereas training uses finite differences, LoRA, and the sample-dependent scale~\eqref{eq:scale}. With multiple rewards,~\eqref{eq:J-improve} applies only to the combined score. We therefore test the main predictions in Section~\ref{sec:ablation}: the need for the correction and positive-length intervals, the absence of an update under constant advantages, and consistency across sampling budgets. Appendix~\ref{app:vr} gives an additional variance-reduction result.

\section{Experiments}
\label{sec:expts}

We evaluate MFA as on-policy RL on SD3.5-Medium (Sections~\ref{sec:expts-sd35}--\ref{sec:ablation}). Appendix~\ref{sec:expts-dna} applies the same objective to DNA promoters, both as teacher-free on-policy RL for a generator defined on a manifold and as reward-graded distillation.

\begin{table}[t]
\caption{Text-to-image results on SD3.5-Medium. All rows without a mark are copied from Table~1 of MeanFlowNFT~\citep{huang2026meanflownft}, which evaluates at $1024\times1024$ following DiffusionNFT. $^\dagger$Evaluated with our pipeline (same prompts, resolution, seeds, and scorers as MFA). Among few-step models, \textbf{bold} marks the best value and \underline{underline} the second best. PickScore is the raw logit.}
\label{tab:baselines}

\centering
\setlength{\tabcolsep}{4pt}
\renewcommand{\arraystretch}{1.08}

\resizebox{\linewidth}{!}{%
\begin{tabular}{lcccccccc}
\toprule
Method & ImageReward$\uparrow$ & CLIPScore$\uparrow$ & Aesthetic$\uparrow$ & PickScore$\uparrow$ & HPSv2$\uparrow$ & HPSv3$\uparrow$ & GenEval2$\uparrow$ & OCR$\uparrow$ \\
\midrule

\rowcolor{gray!15}
\multicolumn{9}{c}{\textit{Multi-step models} {\color{gray}\textit{(40 steps)}}} \\

SD3.5-M$^\dagger$ & $-0.463$ & $0.240$ & $5.178$ & $20.77$ & $0.208$ & $2.769$ & $0.104$ & $0.134$ \\
\quad + DiffusionNFT$^\dagger$~\citep{zheng2026diffusionnft} & $1.426$ & $0.299$ & $5.492$ & $23.55$ & $0.330$ & $13.893$ & $0.224$ & $0.639$ \\

\midrule
\rowcolor{gray!15}
\multicolumn{9}{c}{\textit{Few-step models} {\color{gray}\textit{(4 steps)}}} \\

DMD~\citep{yin2024dmd} & $0.924$ & $0.284$ & $5.506$ & $22.28$ & $0.287$ & $11.716$ & $0.204$ & $0.400$ \\
CDM~\citep{liu2026cdm} & $1.031$ & $0.282$ & $5.572$ & $22.42$ & $0.298$ & $12.519$ & $0.202$ & $0.323$ \\
AnyFlow$^\dagger$~\citep{gu2026anyflow} & $1.113$ & $0.289$ & $5.420$ & $22.48$ & $0.297$ & $12.069$ & $0.190$ & $0.452$ \\

\cmidrule(lr){1-9}

RTDMD~\citep{huang2026rtdmd} & $1.232$ & $0.278$ & $\underline{6.129}$ & $23.28$ & $\underline{0.327}$ & $\underline{13.925}$ & $0.204$ & $0.297$ \\
Rdm~\citep{fan2026rdm} & $0.724$ & $0.272$ & $5.754$ & $22.07$ & $0.276$ & $11.212$ & $0.162$ & $0.376$ \\
DMD + DiffusionNFT & $0.716$ & $0.284$ & $5.378$ & $21.96$ & $0.271$ & $9.699$ & $0.225$ & $0.487$ \\
CDM + DiffusionNFT & $0.146$ & $0.275$ & $4.834$ & $21.38$ & $0.214$ & $-3.283$ & $0.210$ & $0.330$ \\
AnyFlow + DiffusionNFT & $1.239$ & $0.292$ & $5.949$ & $23.09$ & $0.292$ & $12.138$ & $\underline{0.234}$ & $0.595$ \\
TDM-R1$^\dagger$~\citep{luo2026tdmr1} & $1.444$ & $0.293$ & $5.942$ & $22.80$ & $0.324$ & $13.512$ & $0.219$ & $0.623$ \\
MeanFlowNFT$^\dagger$~\citep{huang2026meanflownft} & $\underline{1.446}$ & $0.296$ & $5.928$ & $\underline{23.50}$ & $\underline{0.327}$ & $13.676$ & $0.218$ & $\underline{0.627}$ \\

\rowcolor{blue!7}
\textbf{MeanFlowAdvantage (ours)} & $\mathbf{1.451}$ & $\mathbf{0.297}$ & $\mathbf{6.300}$ & $\mathbf{23.63}$ & $\mathbf{0.330}$ & $\mathbf{13.933}$ & $\underline{0.229}$ & $\mathbf{0.631}$ \\

\bottomrule
\end{tabular}%
}
\end{table}

\begin{figure}[t]
\centering
\includegraphics[width=0.85\linewidth]{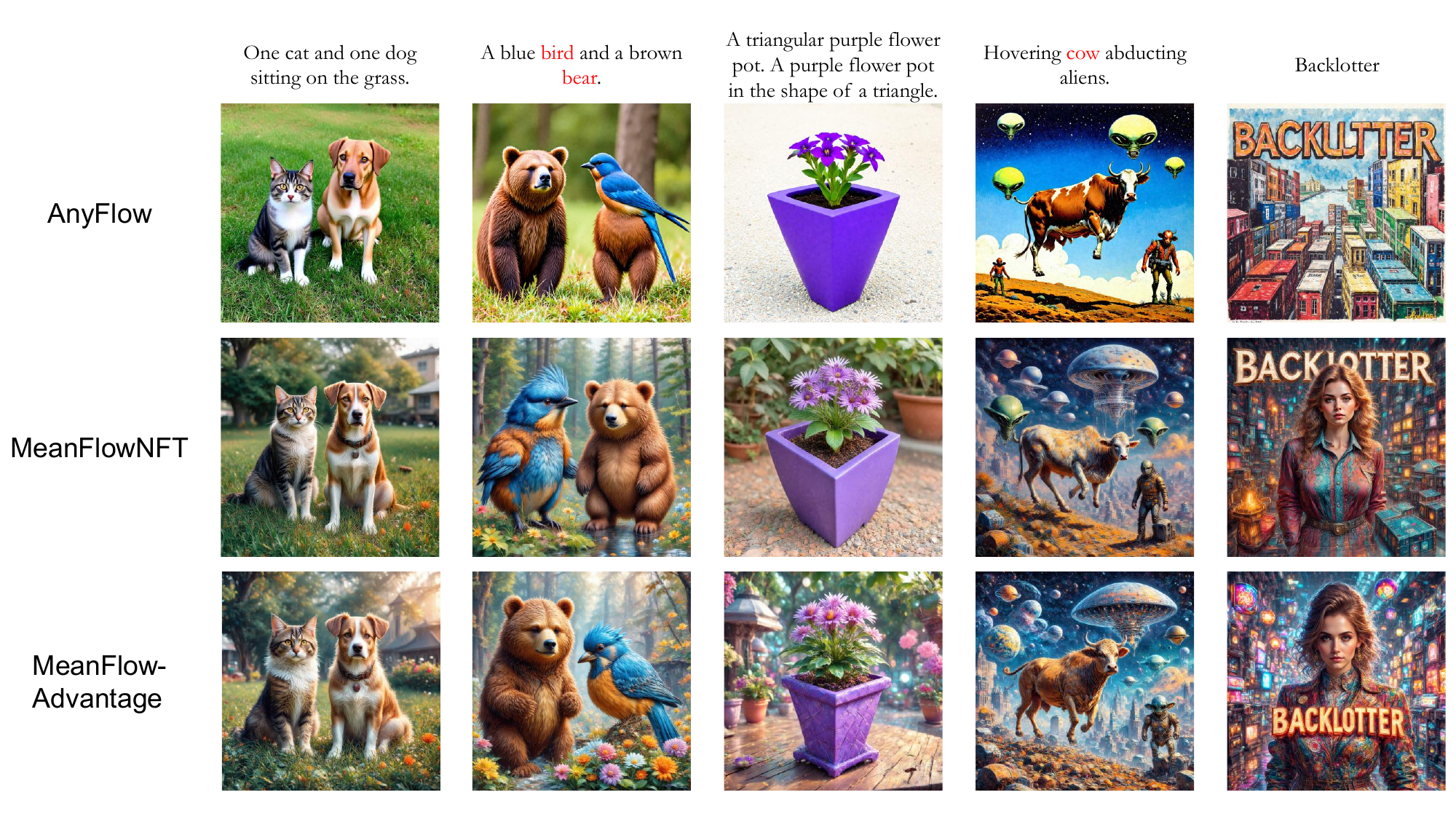}
\vspace{-3mm}
\caption{Qualitative comparison on prompts from GenEval, OCR, and DrawBench.}
\label{fig:qualitative}
\vspace{2mm}
\includegraphics[width=0.85\linewidth]{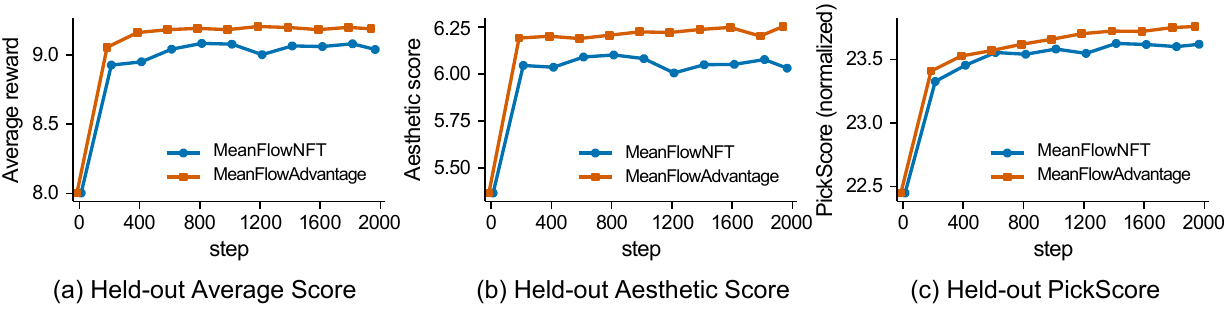}
\vspace{-3mm}
\caption{Held-out DrawBench scores during training, MeanFlowAdvantage versus MeanFlowNFT under the same outer loop.}
\label{fig:image-training}
\end{figure}

\subsection{Text-to-image generation with SD3.5-Medium}
\label{sec:expts-sd35}

\textbf{Setup.} We start from the merged AnyFlow pretraining and on-policy distillation adapters for SD3.5-Medium~\citep{esser2024sd3,gu2026anyflow} and train a fresh LoRA with four-step, CFG-free rollouts. Training uses $512\times512$ images, $L=48$ prompt groups with $K=24$ images each, and AdamW with learning rate $3\times10^{-6}$. PickScore, HPSv2, and CLIPScore are the training rewards, with equal weights~\citep{kirstain2023pickscore,wu2023hpsv2,hessel2021clipscore}. ImageReward, Aesthetic Score, HPSv3, GenEval2, and OCR are evaluated but never optimized~\citep{xu2023imagereward}. DrawBench~\citep{saharia2022imagen} is used for online evaluation. The outer loop and hyperparameters match MeanFlowNFT, and Appendix~\ref{app:expts} lists them all.

\textbf{Main results.} Table~\ref{tab:baselines} compares MFA with the published baselines and with five matched evaluations ($^\dagger$). Under the matched evaluation, MFA improves on both MeanFlowNFT and TDM-R1 on all eight metrics. Relative to MeanFlowNFT, the largest gains are on Aesthetic Score ($+0.37$) and HPSv3 ($+0.26$), neither of which is a training reward, indicating that the improvement is not confined to the optimized objectives. Among all few-step models, MFA ranks first on seven of the eight metrics, with particularly strong results on Aesthetic Score, PickScore, HPSv2, HPSv3, and OCR. With four NFEs, MFA exceeds the 40-step DiffusionNFT on ImageReward, CLIPScore, Aesthetic Score, HPSv2, and HPSv3 and is within $0.01$ on PickScore, at one tenth of the sampling cost. DiffusionNFT remains better on GenEval2 and OCR, both of which its multi-reward setup trains on directly. Qualitatively, MFA more often preserves the requested object count, attribute binding, and text (Figure~\ref{fig:qualitative}). Figure~\ref{fig:image-training} further shows that these gains persist through most of training rather than arising from a single selected checkpoint.

\begin{figure}[t]
\centering

\begin{minipage}{\linewidth}
\centering
\captionof{table}{Any-step evaluation on DrawBench.}
\label{tab:mfa-m2-any-step}
\small
\resizebox{0.75\linewidth}{!}{%
\begin{tabular}{rcccccccc}
\toprule
NFE & IR$\uparrow$ & CLIP$\uparrow$ & Aes$\uparrow$ & Pick$\uparrow$ & HPSv2$\uparrow$ & HPSv3$\uparrow$ & GenEval2$\uparrow$ & OCR$\uparrow$ \\
\midrule
1  & 0.489 & 0.271 & 5.675 & 21.692 & 0.248 & 4.533  & 0.160 & 0.140 \\
2  & 1.377 & 0.288 & $\mathbf{6.359}$ & 23.328 & 0.316 & 12.602 & 0.225 & 0.532 \\
4  & 1.451 & 0.297 & 6.300 & 23.637 & 0.330 & 13.933 & 0.229 & $\mathbf{0.631}$ \\
8  & 1.458 & $\mathbf{0.298}$ & 6.174 & $\mathbf{23.713}$ & $\mathbf{0.333}$ & 14.218 & 0.236 & 0.625 \\
16 & 1.467 & $\mathbf{0.298}$ & 6.084 & 23.679 & 0.332 & $\mathbf{14.255}$ & 0.240 & 0.621 \\
32 & $\mathbf{1.492}$ & 0.297 & 6.045 & 23.646 & 0.332 & 14.207 & $\mathbf{0.250}$ & 0.625 \\
\bottomrule
\end{tabular}}
\end{minipage}

\vspace{2mm}

\begin{minipage}{\linewidth}
\centering
\includegraphics[width=\linewidth]{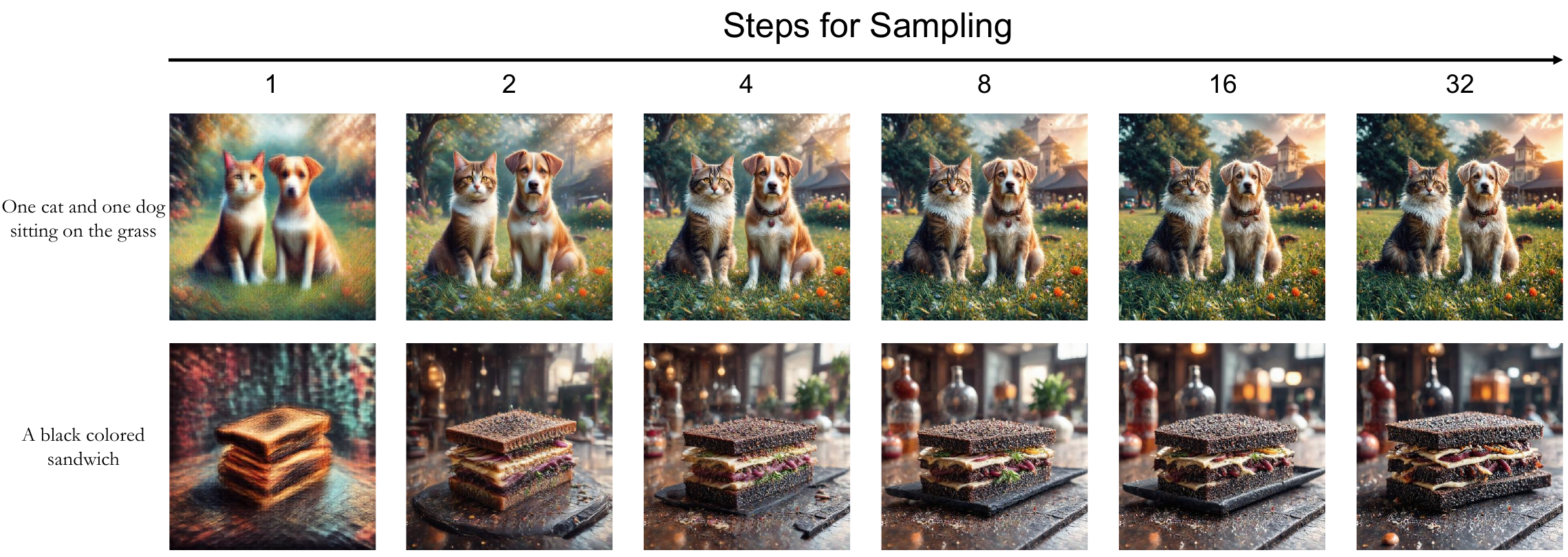}
\captionof{figure}{Qualitative test-time scaling of MeanFlowAdvantage across sampling budgets.}
\label{fig:qualitative_tts}
\end{minipage}

\vspace{2mm}

\begin{minipage}{\linewidth}
\centering
\includegraphics[width=0.85\linewidth]{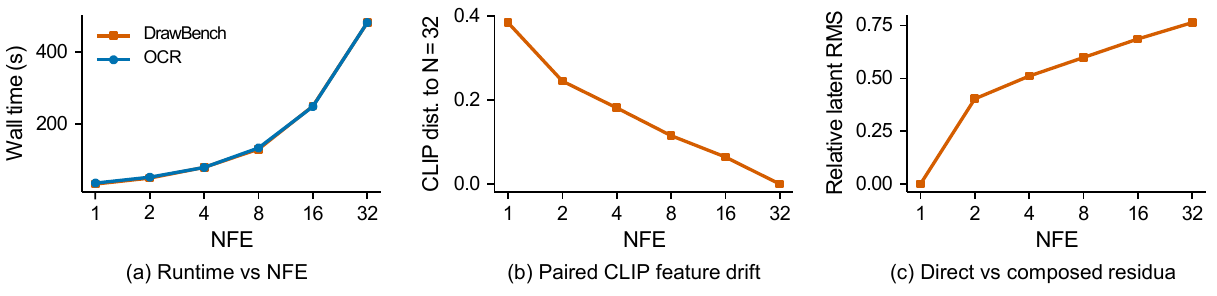}
\captionof{figure}{Any-step behavior of MFA: runtime, drift of samples across budgets, and semigroup residual of the learned flow map.}
\label{fig:any-step-analysis-placeholder}
\end{minipage}

\end{figure}

\textbf{Any-step behavior.} Theorem~\ref{thm:transfer} predicts that, near its fixed point, one model should serve every sampling budget. Table~\ref{tab:mfa-m2-any-step} tests this with the checkpoint trained only with four-step rollouts. One step remains clearly harder. Two steps recover most of the quality, and from four to 32 steps the scores stay on a plateau: ImageReward, GenEval2, and HPSv3 keep improving slightly, while Aesthetic Score peaks at two steps. The model is therefore usable across budgets but is not an exact flow map. The semigroup probe in Figure~\ref{fig:any-step-analysis-placeholder} measures this directly. It compares the direct map $\Phi_{1\to0}$ with compositions of $N$ sub-steps (Appendix~\ref{app:semigroup}); the relative residual grows from $0.40$ at $N=2$ to $0.76$ at $N=32$. The fixed point~\eqref{eq:fixed-point} is thus only approximately reached. Four steps give a good quality-cost trade-off. Figure~\ref{fig:qualitative_tts} shows samples across budgets.

\begin{figure}[t]
\centering
\includegraphics[width=0.85\linewidth]{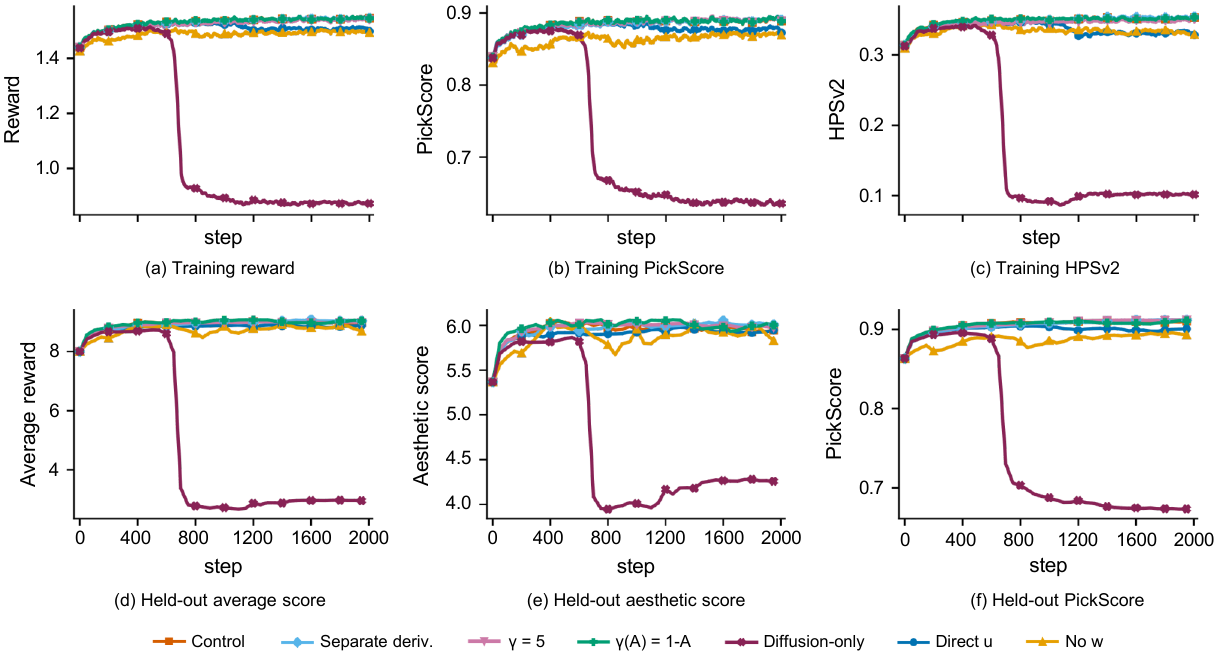}
\caption{Training reward and held-out score of the matched MFA ablations. Boundary-only training collapses late; the positive-only and $A\equiv1$ runs collapse within the first few dozen steps and are stopped at step 300.}
\label{fig:ablation-curves-placeholder}
\end{figure}

\subsection{Testing the analysis on the image model}
\label{sec:ablation}

Each ablation changes one component of MFA under the same SD3.5-Medium initialization, fresh LoRA, 48 prompt groups per update, 2000 steps, and seed 42. Figure~\ref{fig:ablation-curves-placeholder} shows the training curves, and Appendix~\ref{app:image-ablation-table} reports the endpoints (Table~\ref{tab:mfa-ablation}). The results broadly follow the predictions of Section~\ref{sec:analysis}. \textbf{(i) The MeanFlow correction matters.} Removing the correction while keeping positive-length intervals lowers both the training reward ($1.507$ vs.\ $1.549$) and held-out aggregate ($8.89$ vs.\ $8.98$), while roughly doubling the gradient norm. Computing a separate derivative for each network reaches similar quality ($9.04$) but costs 23\% more per step, suggesting that sharing is mainly an efficiency choice in practice, while its theoretical role is to preserve the equivalence in Theorem~\ref{thm:tilt}. \textbf{(ii) Positive-length intervals are necessary.} Boundary-only training constrains the instantaneous velocity but leaves the finite-interval maps used by the sampler unconstrained. The run collapses after about 670 steps, ending with a held-out aggregate of $2.96$ and a gradient norm above $400$, consistent with Theorem~\ref{thm:transfer}. \textbf{(iii) Signed advantages provide the direction.} With $A\equiv1$, Proposition~\ref{prop:step} predicts no reward-directed update, and training collapses within 20 steps; keeping only positive advantages also collapses early (step 22), suggesting that negative samples matter beyond the population reweighting picture. In contrast, increasing $\gamma$ from $1.1$ to $5$ or using $\gamma(A)=1-A$ gives endpoints close to the control, so the gain does not appear to come from a narrowly tuned trust region. \textbf{(iv) Adaptive scaling stabilizes optimization.} Removing $w$ from~\eqref{eq:scale} raises the final gradient norm from $0.97$ to $11.97$ and lowers the held-out aggregate to $8.70$ without immediate collapse. Since $w$ leaves the per-sample minimizer unchanged, its effect is on optimization stability rather than on the target itself.

\subsection{Reward-filtered advantage-weighted promoter fine-tuning}
\label{sec:promoter-main}

We further test MFA on 1024-bp FANTOM5 promoters with a Riemannian MeanFlow generator and negative Sei profile MSE as reward. In teacher-free on-policy RL, MFA lowers the pretrained Sei MSE from $0.0714$ to $0.0482$ over three seeds and is statistically indistinguishable from MeanFlowNFT on the optimized reward, consistent with Proposition~\ref{prop:nft}. For distillation, a frozen 10-step Sei-guided teacher provides improved endpoints; we keep pairs with $\Delta r\ge0.002$ and use $A=\operatorname{clip}(10\Delta r,0,1)$. On identical pairs and advantages, MFA reaches $0.0467$ Sei MSE versus $0.0570$ for MeanFlowNFT, an 18\% reduction. Full configurations and ablations are in Appendix~\ref{sec:expts-dna}.

\section{Conclusion}
\label{sec:conclusion}

We introduced MeanFlowAdvantage, a signed advantage-weighted objective that aligns reward fine-tuning with the average-velocity representation deployed by MeanFlow. A single detached correction shared by the learner, rollout, and reference networks turns the prediction-space quadratic into an exact advantage-weighted MeanFlow regression, while keeping the rollout and reference anchors directly on the deployed finite-interval maps. At the population level, this objective is equivalent to fitting the MeanFlow target of a reward-reweighted distribution; its update direction is governed by an advantage--target covariance, and self-consistent fixed points define exact flow maps that transfer across sampling budgets. Empirically, MFA improves all eight reported metrics over matched MeanFlowNFT on SD3.5-Medium and matches or exceeds the 40-step DiffusionNFT baseline on six of eight with only four NFEs. The same formulation also extends beyond images, supporting both teacher-free RL and reward-graded distillation for DNA promoter generation. Open questions remain around approximate flow-map consistency, principled negative advantages for distillation, and the three-seed sequence-fidelity trend.

\subsection*{AI Use Statement}
Generative AI tools were used to assist with manuscript editing, language refinement, LaTeX formatting, and mathematical consistency checking. They were not used as a substitute for experimental validation or independent verification of the reported results. The authors reviewed and verified all AI-assisted text, derivations, implementation details, references, and experimental claims, and take full responsibility for the final content of the paper.

\subsection*{Ethics statement}
The models inherit biases and failure modes of their pretrained generators and reward models. The promoter experiments evaluate public human genomic sequences computationally; Sei scores are not experimental validation of regulatory function.

\subsection*{Reproducibility Statement}
The main algorithm and optimization objective are specified in Algorithm~\ref{alg:mfa} and Section~\ref{sec:method}. Experimental configurations and implementation details are provided in Appendix~\ref{app:expts} and Appendix~\ref{app:impl}, complete proofs and theoretical assumptions are given in Appendix~\ref{app:proofs}, and additional image-model ablations and promoter evaluations are reported in Appendices~\ref{app:image-ablation-table} and~\ref{app:dna-extra}. The code is available on \href{https://github.com/HaCTang/MeanFlowAdvantage}{https://github.com/HaCTang/MeanFlowAdvantage}, and the model weights are available on \href{https://huggingface.co/HaC-RL/MeanFlowAdvantage}{https://huggingface.co/HaC-RL/MeanFlowAdvantage}.

\bibliography{iclr2027_conference}
\bibliographystyle{iclr2027_conference}
\clearpage

\appendix
\setcounter{figure}{0}
\renewcommand{\thefigure}{S\arabic{figure}}
\renewcommand{\theHfigure}{S\arabic{figure}}
\setcounter{table}{0}
\renewcommand{\thetable}{S\arabic{table}}
\renewcommand{\theHtable}{S\arabic{table}}

\section{Proofs for Section~\ref{sec:analysis}}
\label{app:proofs}\label{app:analysis-extra}

\textbf{Setting.} Throughout, $c$ is drawn from the prompt distribution, $x_0\sim\piold(\cdot\mid c)$, $\epsilon\sim\mathcal N(0,I)$, and $(s,t)$ is drawn independently of $(x_0,\epsilon)$. We write $x_t=(1-t)x_0+t\epsilon$, $v_t=\epsilon-x_0$, and $Z=(x_t,s,t,c)$. The advantage is a function $A=A(x_0,c)$. Under (A1), $d=D(Z)+J(Z)\,v_t$ with $D=\partial_tu_{\sold}(Z)$ and $J=\partial_xu_{\sold}(Z)$; both are functions of $Z$ only. All second moments are assumed finite. Conditional expectations under $\piold$ are written $\mathbb E[\cdot\mid Z]$.

\subsection{Average-velocity form of the loss and boundary reduction}
\label{app:reparam}\label{app:reduce}

By~\eqref{eq:V}, $F_\theta=x_t-tu_\theta-t(t-s)d$. Since $x_t-x_0=t\,v_t$,
\begin{equation}
F_\theta-x_0=t\,v_t-t\,u_\theta-t(t-s)d=-t\bigl(u_\theta-\bar u\bigr),\qquad \bar u=v_t-(t-s)d .
\end{equation}
The anchor residuals follow from~\eqref{eq:V} because the correction terms cancel: $F_\theta-F_k=-t(u_\theta-u_k)$. Substituting both into~\eqref{eq:mfa-loss} gives~\eqref{eq:mfa-u}. This identity holds for any value of $d$, including the finite difference. At $s=t$, the factor $t-s$ vanishes, so $\bar u=v_t$ and $F_k=x_t-tu_k(x_t,t,t)$ for every network. The loss~\eqref{eq:mfa-u} then equals the AdvantageFlow loss~\eqref{eq:af-loss} applied to the boundary predictor $x_t-tu_\theta(x_t,t,t)$, for the same samples, advantages, and coefficients. The scaled loss~\eqref{eq:scale} reduces in the same way when the comparison uses the same denominator.

\subsection{Proof of Theorem~\ref{thm:tilt}}
\label{app:tilt-proof}

\textbf{Step 1: the rollout anchor is a MeanFlow regression.} Let $g(Z)$ be any function of $Z$. Expanding the square and using (A3),
\begin{equation}
\mathbb E\bigl[\|g-\bar u\|_2^2\mid Z\bigr]
=\|g-u_{\sold}\|_2^2+\mathbb E\bigl[\|u_{\sold}-\bar u\|_2^2\mid Z\bigr]
+2\bigl\langle g-u_{\sold},\,u_{\sold}-\mathbb E[\bar u\mid Z]\bigr\rangle ,
\end{equation}
and the last term is zero. By~\eqref{eq:shared-cd}, the learner output $u_\theta(Z)$ is such a function, so
\begin{equation}
t^2\|u_\theta-u_{\sold}\|_2^2=\mathbb E\bigl[t^2\|u_\theta-\bar u\|_2^2\mid Z\bigr]-\mathbb E\bigl[t^2\|u_{\sold}-\bar u\|_2^2\mid Z\bigr].
\label{eq:anchor-as-regression}
\end{equation}
Taking expectations in~\eqref{eq:mfa-u} with $\lambda=0$ gives
\begin{equation}
\mathcal L_{\mathrm{MFA}}(\theta)=\mathbb E\bigl[(A+\gamma)\,t^2\|u_\theta-\bar u\|_2^2\bigr]+C,\qquad C=-\gamma\,\mathbb E\bigl[t^2\|u_{\sold}-\bar u\|_2^2\bigr].
\end{equation}

\textbf{Step 2: change of measure.} Let $h(x_0,\epsilon,s,t,c)=t^2\|u_\theta(Z)-\bar u\|_2^2$. Because $A$ depends only on $(x_0,c)$ and $(\epsilon,s,t)$ are independent of $x_0$,
\begin{equation}
\mathbb E\bigl[(A+\gamma)h\bigr]
=\mathbb E_c\Bigl[\int(\gamma+A)\,\piold(x_0\mid c)\,\mathbb E_{\epsilon,s,t}[h]\,\mathrm dx_0\Bigr]
=\mathbb E_c\Bigl[(\gamma+\bar A_c)\,\mathbb E_{q}\bigl[h\mid c\bigr]\Bigr],
\end{equation}
where $q$ is defined in~\eqref{eq:tilt}. By (A2), $q$ is a valid probability distribution. This proves~\eqref{eq:thm-tilt}.

\textbf{Step 3: minimizer.} For $\lambda\ge0$, the conditional loss given $Z$ is
\begin{equation}
t^2\Bigl(\mathbb E\bigl[A\|u-\bar u\|_2^2\mid Z\bigr]+\gamma\|u-u_{\sold}\|_2^2+\lambda\|u-u_{\sref}\|_2^2\Bigr),
\end{equation}
a quadratic in $u$ with Hessian $2t^2(\mathbb E[A\mid Z]+\gamma+\lambda)I$. This is positive definite for $t>0$ by (A2). Its unique minimizer is
\begin{equation}
u^+=\frac{\mathbb E[A\bar u\mid Z]+\gamma u_{\sold}+\lambda u_{\sref}}{\mathbb E[A\mid Z]+\gamma+\lambda}.
\label{eq:app-min}
\end{equation}
By (A3), $\mathbb E[A\bar u\mid Z]+\gamma u_{\sold}=\mathbb E[(A+\gamma)\bar u\mid Z]$. Under $q$, the joint density of $(x_0,\epsilon,s,t)$ given $c$ is $(\gamma+A)/(\gamma+\bar A_c)$ times that under $\piold$. By Bayes' rule, the conditional law given $Z$ is reweighted by $\gamma+A$; the normalizer $\gamma+\bar A_c$ depends only on $c$, which is part of $Z$. Hence
\begin{equation}
\mathbb E[(A+\gamma)\bar u\mid Z]=\mathbb E[A+\gamma\mid Z]\;\mathbb E_q[\bar u\mid Z].
\end{equation}
Since $D$ and $J$ are functions of $Z$, $\mathbb E_q[\bar u\mid Z]=\mathbb E_q[v_t\mid Z]-(t-s)\bigl(D+J\,\mathbb E_q[v_t\mid Z]\bigr)$. Because $(s,t)$ is independent of $(x_0,\epsilon)$, $\mathbb E_q[v_t\mid Z]=v_q(x_t,t,c)$, so $\mathbb E_q[\bar u\mid Z]=u_q$. Substituting into~\eqref{eq:app-min} gives $u^+=\alpha u_q+(1-\alpha)u_{\sref}$ with the stated $\alpha$. \hfill$\square$

\textbf{Why the correction must be shared.} Suppose the learner and the rollout used separate corrections $d_\theta$ and $d_{\sold}$. The rollout residual would be $u_\theta-u_{\sold}+(t-s)(d_\theta-d_{\sold})$, and under (A1) $d_\theta-d_{\sold}=\partial_t(u_\theta-u_{\sold})+\partial_x(u_\theta-u_{\sold})\,v_t$. This residual depends on $v_t$ and not only on $Z$, so Step~1 fails: the cross term no longer vanishes, and the anchor is no longer a MeanFlow regression on rollout samples.

\subsection{Proof of Proposition~\ref{prop:step}}
\label{app:step-proof}

Subtracting $u_{\sold}$ from~\eqref{eq:app-min},
\begin{equation}
u^+-u_{\sold}=\frac{\mathbb E[A\bar u\mid Z]-\mathbb E[A\mid Z]\,u_{\sold}+\lambda(u_{\sref}-u_{\sold})}{\mathbb E[A\mid Z]+\gamma+\lambda}.
\end{equation}
By (A3), $\mathbb E[A\bar u\mid Z]-\mathbb E[A\mid Z]\,\mathbb E[\bar u\mid Z]=\operatorname{Cov}(A,\bar u\mid Z)$, which gives the identity in~\eqref{eq:step}. For the bound, take any unit vector $e$. By Cauchy--Schwarz, $e^\top\operatorname{Cov}(A,\bar u\mid Z)=\operatorname{Cov}(A,e^\top\bar u\mid Z)\le\sigma_A\sqrt{\operatorname{Var}(e^\top\bar u\mid Z)}\le\sigma_A\sigma_{\bar u}$. Taking the supremum over $e$ bounds the norm of the covariance vector. The triangle inequality and the positivity of the denominator (A2) complete the proof. \hfill$\square$

\subsection{Proof of Proposition~\ref{prop:nft} and the relation to prediction-space NFT}
\label{app:nft-proof}\label{sec:nft-link}\label{sec:nft}

MeanFlowNFT forms the implicit positive and negative velocities $V^\pm=V_{\sold}\pm\beta(V_\theta-V_{\sold})$ and minimizes $r\|V^+-v_t\|_2^2+(1-r)\|V^--v_t\|_2^2$. With the shared correction, $V_\theta-V_{\sold}=u_\theta-u_{\sold}=:\Delta$, a function of $Z$, and $V_{\sold}-v_t=u_{\sold}-\bar u=:e$. Hence
\begin{equation}
r\|\beta\Delta+e\|_2^2+(1-r)\|{-\beta\Delta}+e\|_2^2=\beta^2\|\Delta\|_2^2+2\beta A\langle\Delta,e\rangle+\|e\|_2^2,\qquad A=2r-1 .
\label{eq:nft-expand}
\end{equation}
Conditioning on $Z$ and minimizing over $\Delta$ gives $\Delta=-\mathbb E[Ae\mid Z]/\beta=\bigl(\mathbb E[A\bar u\mid Z]-\mathbb E[A\mid Z]u_{\sold}\bigr)/\beta$. By (A3), this equals $\operatorname{Cov}(A,\bar u\mid Z)/\beta$, which proves~\eqref{eq:nft-min}. \hfill$\square$

\textbf{The positive-policy example.} Let $A=2r-1$ and $\gamma=1$. Then $\gamma+A=2r$ and $q\propto r\,\piold=\pi^+$. Write $\rho=\mathbb E[r\mid Z]$ and $u_{\pi^+}=\mathbb E_{\pi^+}[\bar u\mid Z]$. By (A3), $\operatorname{Cov}(A,\bar u\mid Z)=2\bigl(\mathbb E[r\bar u\mid Z]-\rho\,u_{\sold}\bigr)=2\rho\,(u_{\pi^+}-u_{\sold})$. The NFT minimizer is therefore $u_{\sold}+(2\rho/\beta)(u_{\pi^+}-u_{\sold})$. It equals $u_{\pi^+}$ only where $2\rho(Z)=\beta$. By Theorem~\ref{thm:tilt}, the MFA minimizer is $u_{\pi^+}$ at every $Z$.

\textbf{Prediction-space form.} Multiplying by $t^2$ and dropping $\theta$-independent terms, \eqref{eq:nft-expand} is equivalent to $\beta A\|F_\theta-x_0\|_2^2+\beta(\beta-A)\|F_\theta-F_{\sold}\|_2^2$. After dividing by $\beta$, NFT is the MFA quadratic with the sample-dependent anchor weight $\gamma(A)=\beta-A$ and constant total curvature $\beta$. At $\beta=1$, it matches the variant $\gamma(A)=1-A$ with $\lambda=0$. Separate adaptive denominators for the positive and negative branches break this equivalence. The same expansion, with the data target $y_\kappa$, anchor $b$, and reference weight $k$, gives the NFT optimum in~\eqref{eq:dna-opt}.

\subsection{Flow-map consistency and proof of Theorem~\ref{thm:transfer}}
\label{app:consistency}\label{app:transfer}

For a velocity field $v$, define $\mathcal I_v[u]=u+(t-s)\bigl(\partial_tu+(\partial_xu)v\bigr)$.

\begin{lemma}[Flow-map consistency]
\label{lem:consistency}
Suppose $v$ admits a unique flow on $[0,1]$, and $u$ is continuously differentiable for $s<t$ and bounded as $t\downarrow s$. If $\mathcal I_v[u]=v$ along every trajectory, then $u(x_t,s,t)=(x_t-x_s)/(t-s)$, where $x_\tau$ is the trajectory of $v$ through $x_t$.
\end{lemma}

\textbf{Proof of the lemma.} Fix $s$ and let $U(\tau)=(\tau-s)\,u(x_\tau,s,\tau)$ for $\tau>s$. By the chain rule and the hypothesis, $\mathrm dU/\mathrm d\tau=\mathcal I_v[u](x_\tau,s,\tau)=v(x_\tau,\tau)$. Boundedness of $u$ gives $U(\tau)\to0$ as $\tau\downarrow s$. Integrating from $s$ to $t$ yields $(t-s)\,u(x_t,s,t)=\int_s^tv(x_\tau,\tau)\,\mathrm d\tau=x_t-x_s$. \hfill$\square$

\textbf{Proof of Theorem~\ref{thm:transfer}.} Condition~\eqref{eq:fixed-point} says $\mathcal I_{v_q}[u^\ast]=v_q$. By Lemma~\ref{lem:consistency}, $x_t-(t-s)u^\ast(x_t,s,t)=x_s$ on every trajectory of $v_q$, so each step of~\eqref{eq:nstep} with $u^\ast$ is the exact flow map of $v_q$ over its interval. A composition over any grid is therefore the exact flow from $t=1$ to $t=0$. The rectified-flow velocity $v_q$ transports $\mathcal N(0,I)$ at $t=1$ to $q$ at $t=0$, so the sampler draws from $q$ for every $N$. For the reward,
\begin{equation}
\mathbb E_q[R\mid c]-\mathbb E_{\piold}[R\mid c]
=\frac{\mathbb E_{\piold}[(\gamma+A)R\mid c]}{\gamma+\bar A_c}-\mathbb E_{\piold}[R\mid c]
=\frac{\operatorname{Cov}_{\piold}(A,R\mid c)}{\gamma+\bar A_c}.
\end{equation}
If $A=\phi(R)$ with $\phi$ nondecreasing, Chebyshev's association inequality gives $\operatorname{Cov}(\phi(R),R)\ge0$. The clipped, prompt-centered advantage~\eqref{eq:adv} is of this form for a scalar reward when the prompt mean and the global scale are treated as fixed. \hfill$\square$

\subsection{Variance reduction of the shared anchor}
\label{app:vr}

\begin{proposition}[Shared-anchor variance reduction]
\label{prop:vr}
Freeze the rollout model and the shared correction, and let $J_\theta=-t\,\partial_\theta u_\theta(Z)$. If $\mathbb E[F_{\sold}-x_0\mid Z]=0$, then
\begin{equation}
\mathbb E\|F_\theta-x_0\|_2^2=\mathbb E\|F_\theta-F_{\sold}\|_2^2+C
\label{eq:vr}
\end{equation}
for a $\theta$-independent constant $C$. Moreover, the anchor gradient $2J_\theta^\top(F_\theta-F_{\sold})$ is the conditional expectation, given $Z$, of the regression gradient $2J_\theta^\top(F_\theta-x_0)$, so its covariance is no larger in the positive-semidefinite order.
\end{proposition}

\textbf{Proof.} The condition is (A3) written in prediction space, since $F_{\sold}-x_0=-t(u_{\sold}-\bar u)$. Let $e=F_{\sold}-x_0$. By~\eqref{eq:shared-cd}, $F_\theta-F_{\sold}$ is a function of $Z$, so $\mathbb E\langle F_\theta-F_{\sold},e\rangle=\mathbb E\langle F_\theta-F_{\sold},\mathbb E[e\mid Z]\rangle=0$. Expanding $\|F_\theta-F_{\sold}+e\|_2^2$ gives~\eqref{eq:vr} with $C=\mathbb E\|e\|_2^2$. For the gradients, let $G=2J_\theta^\top(F_\theta-x_0)$ and $\widetilde G=2J_\theta^\top(F_\theta-F_{\sold})$. Both $J_\theta$ and $F_\theta-F_{\sold}$ are functions of $Z$, so $\mathbb E[G\mid Z]=\widetilde G$. The law of total covariance gives $\operatorname{Cov}(G)=\operatorname{Cov}(\widetilde G)+\mathbb E[\operatorname{Cov}(G\mid Z)]\succeq\operatorname{Cov}(\widetilde G)$. \hfill$\square$

This statement concerns the unscaled rollout anchor only. It does not imply lower variance for the full signed-advantage update with adaptive scaling.

\subsection{Semigroup residual}
\label{app:semigroup}

For $\Phi^u_{s,t}(x)=x-(t-s)u(x,s,t)$, an exact flow map satisfies
\begin{equation}
\Phi^u_{r,t}(x)=\Phi^u_{r,s}\bigl(\Phi^u_{s,t}(x)\bigr),\qquad r<s<t.
\label{eq:semigroup}
\end{equation}
By Lemma~\ref{lem:consistency}, a map that satisfies the fixed point~\eqref{eq:fixed-point} also satisfies~\eqref{eq:semigroup}. The residual of~\eqref{eq:semigroup} therefore measures how far a trained model is from that fixed point. Figure~\ref{fig:any-step-analysis-placeholder} reports the relative latent RMS distance between the direct map $\Phi_{1\to0}$ and uniform $N$-step compositions, which is zero by definition at $N=1$. A small residual is not by itself evidence of reward improvement.

\section{Implementation notes}
\label{app:impl}
\label{sec:impl}

\begin{algorithm}[h]
\caption{One MeanFlowAdvantage update (image model)}
\label{alg:mfa}\label{sec:algo}
\begin{algorithmic}[1]
\Require learner $\theta$, frozen rollout $\theta_{\sold}$ and reference $\theta_{\sref}$; rewards; $\gamma,\lambda$; group size $K$; steps $N$
\State Sample prompts; generate $K$ images per prompt with~\eqref{eq:nstep}
\State Score the images; compute clipped signed advantages with~\eqref{eq:adv}
\For{each detached $(x_0,A,c)$}
 \State Draw $(s,t)$ by~\eqref{eq:three-mode} and $\epsilon$; set $x_t=(1-t)x_0+t\epsilon$
 \State If $s=t$, set $d=0$; otherwise compute $d$ once from $u_{\sold}$ by~\eqref{eq:cd}
 \State Form $V_k=u_k+(t-s)d$ and $F_k=x_t-tV_k$ for $k\in\{\theta,\sold,\sref\}$
 \State Compute $\ell_\theta/w$ with~\eqref{eq:mfa-loss} and~\eqref{eq:scale}
\EndFor
\State Update $\theta$ on the average loss
\State Refresh the rollout model: $\theta_{\sold}\leftarrow\rho\,\theta_{\sold}+(1-\rho)\,\theta$
\end{algorithmic}
\end{algorithm}

\textbf{Scheduler units and finite differences.}
With $T=1000$, $t_{\mathrm{raw}}=Tt$ and $s_{\mathrm{raw}}=Ts$,
the normalized and raw-time expressions agree:
\begin{equation}
 x_t=(1-t)x_0+t\epsilon,\qquad F=x_t-tV,\qquad
 (t_{\mathrm{raw}}-s_{\mathrm{raw}})d_{\mathrm{raw}}=(t-s)D_tu.
 \label{eq:impl-scale}
\end{equation}
The reported finite-difference increment of $5$ is in raw scheduler units, corresponding to $\delta=0.005$. Spatial displacements must therefore use $\delta v_t$, not $5v_t$.
The stencil~\eqref{eq:cd} assumes $s\le t-\delta$ and $t+\delta\le1$. Near a boundary, the stencil must be shortened or replaced by a valid one-sided difference; independently clamping its endpoints while retaining $2\delta$ in the denominator is not the same estimator. On $s=t$, the correction is zero and the stencil is skipped.

\textbf{Advantage and loss scaling.}
The final multi-reward advantage is clipped after aggregation. For constant $\gamma$, a lower bound
$A\ge-\gamma-\lambda+\varepsilon_A$ enforces positive curvature; the reported $\gamma=1.1$, $\lambda=10^{-3}$ and $A\in[-1,1]$ already satisfy it. For $\gamma(A)=1-A$, the curvature is $1+\lambda$ and no additional floor is necessary. The shared denominator in~\eqref{eq:scale} is evaluated in float64. It preserves the relative weights of the three terms within a sample, but changes the relative weights of different samples. It neither equalizes the three residual magnitudes nor preserves the unweighted population regression problem in general.

\textbf{Ablations to isolate the construction.}
The direct-output control uses $F^{\mathrm{direct}}=x_t-tu_\theta$, retaining the same clean target, regularizers, and time distribution. The diffusion-only control sets $\rho_{\mathrm{diff}}=1$. Prediction-space and velocity-space losses must be compared with their effective time weights reported: without other changes they differ by $t^2$. The completed matched endpoints for these controls are reported in Appendix Table~\ref{tab:mfa-ablation}.

\section{Experimental details}
\label{app:expts}

Tables~\ref{tab:hp-sd35}, \ref{tab:hp-dna-onpolicy} and~\ref{tab:hp-dna} collect the reported configurations, and Table~\ref{tab:dna-onpolicy-seeds} the per-seed DNA RL results. The image baseline follows the published Stage-3 protocol. Both methods use $K=24$ samples per prompt; the corresponding EMA schedule, evaluation resolution, prompt lists, and checkpoint identifiers are fixed according to the reported Stage-3 protocol.

\begin{table}[H]
\caption{SD3.5 Medium RL (Stage~3 protocol of MeanFlowNFT).
Sampling and reward choices follow the baseline; the objectives and
regularization coefficients differ.}
\label{tab:hp-sd35}
\begin{center}
{\small
\setlength{\tabcolsep}{3pt}
\begin{tabular}{lcc}
\toprule
 & MeanFlowNFT & MeanFlowAdvantage \\
\midrule
Init & AnyFlow + fresh LoRA & AnyFlow + fresh LoRA \\
LoRA rank / $\alpha$ & $32$ / $64$ & $32$ / $64$ \\
Image size (train) & $512$ & $512$ \\
Rollout NFE / CFG & $4$ / none & $4$ / none \\
Prompt groups $L$ per update & $48$ & $48$ \\
Optimizer & AdamW & AdamW \\
Learning rate & $3\times 10^{-6}$ & $3\times 10^{-6}$ \\
$(s,t)$ mix $(\rho_{\mathrm{diff}},\rho_{\mathrm{cons}})$ & $(0.5,0.25)$ & $(0.5,0.25)$ \\
Central-difference increment (raw units) & $5$ & $5$ \\
Shared $D_t$ & yes & yes \\
Train rewards & Pick / HPS / CLIP & Pick / HPS / CLIP \\
$\beta$ (NFT) & $0.1$ & --- \\
$\gamma$ / $\lambda$ (LS) & --- & $1.1$ / $10^{-3}$ \\
Reference coefficient (objective-specific) & $10^{-4}$ & $10^{-3}$ \\
Advantage clip & $[-1,1]$ & $[-1,1]$ \\
\bottomrule
\end{tabular}}
\end{center}
\end{table}

\begin{table}[H]
\caption{FANTOM5 on-policy DNA RL (no teacher). Both methods share every
row above the line; only the loss coefficients differ.}
\label{tab:hp-dna-onpolicy}
\begin{center}
{\small
\begin{tabular}{lcc}
\toprule
 & MeanFlowNFT & MeanFlowAdvantage \\
\midrule
Init & RMF \texttt{ckpts/0120} & RMF \texttt{ckpts/0120} \\
Signals per update & $8$ & $8$ \\
Group size $K$ & $8$ & $8$ \\
Rollout NFE & $1$ & $1$ \\
$(s,t)$ mix (boundary, $s{=}1$, general) & $(0.5,0.25,0.25)$ & $(0.5,0.25,0.25)$ \\
$t$ sampler & logit-normal$(-0.4,1)$ & logit-normal$(-0.4,1)$ \\
Noise draws per sample & $2$ & $2$ \\
Shared correction (JVP) & yes & yes \\
Tangent clip & $100$ & $100$ \\
Adaptive scale exponent $p$ & $0.5$ & $0.5$ \\
Learning rate & $10^{-4}$ & $10^{-4}$ \\
Grad clip & $0.5$ & $0.5$ \\
Rollout EMA $\rho$ & $0.9$ & $0.9$ \\
Advantage clip & $[-1,1]$ & $[-1,1]$ \\
Train steps (eval interval) & $600$ ($25$) & $600$ ($25$) \\
Seeds & $123,7,42$ & $123,7,42$ \\
\midrule
$\beta$ / reference weight & $1.0$ / $0.01$ & --- \\
$\gamma$ / $\lambda$ & --- & $1.1$ / $0.01$ \\
\bottomrule
\end{tabular}}
\end{center}
\end{table}

\begin{table}[H]
\caption{Per-seed on-policy DNA results behind Table~\ref{tab:dna}.
Checkpoints are selected on valid ($n{=}256$) and reported on test
($n{=}512$), both with eval seed $7$. $\Delta$ is the paired
NFT$-$MFA test Sei MSE with a $95\%$ sequence bootstrap interval;
positive favors MFA.}
\label{tab:dna-onpolicy-seeds}
\begin{center}
{\small
\begin{tabular}{lcccccc}
\toprule
 & \multicolumn{2}{c}{MeanFlowNFT} & \multicolumn{2}{c}{MeanFlowAdvantage} & \\
\cmidrule(lr){2-3}\cmidrule(lr){4-5}
Seed & Sei MSE$\downarrow$ & $6$-mer$\uparrow$ & Sei MSE$\downarrow$ & $6$-mer$\uparrow$ & $\Delta$ [$95\%$ CI] \\
\midrule
$123$ & $0.0445$ & $0.927$ & $\mathbf{0.0437}$ & $\mathbf{0.952}$ & $0.0008$ $[-0.0031,0.0046]$ \\
$7$ & $\mathbf{0.0526}$ & $0.940$ & $0.0537$ & $\mathbf{0.954}$ & $-0.0012$ $[-0.0051,0.0026]$ \\
$42$ & $\mathbf{0.0441}$ & $0.944$ & $0.0472$ & $\mathbf{0.953}$ & $-0.0031$ $[-0.0069,0.0005]$ \\
\midrule
Mean & $\mathbf{0.0471}$ & $0.937$ & $0.0482$ & $\mathbf{0.953}$ & --- \\
Std. & $0.0048$ & $0.009$ & $0.0051$ & $\mathbf{0.001}$ & --- \\
\bottomrule
\end{tabular}}
\end{center}
\vspace{2pt}
{\footnotesize All three Sei MSE intervals contain zero, so the two
objectives are statistically indistinguishable on the training reward,
as Proposition~\ref{prop:nft} predicts. The $6$-mer ordering is
consistent across seeds but is a three-seed trend, not a tested claim.}
\end{table}

\begin{table}[H]
\caption{FANTOM5 reward-filtered promoter distillation.
Shared rows are identical for both losses; method columns list only
the knobs that differ.}
\label{tab:hp-dna}
\begin{center}
{\small
\begin{tabular}{lcc}
\toprule
 & MeanFlowNFT & MeanFlowAdvantage \\
\midrule
Init & RMF \texttt{ckpts/0120} & RMF \texttt{ckpts/0120} \\
Sequence length & $1024$ bp & $1024$ bp \\
Train / valid / test chr. & not $8$--$10$ / $10$ / $8$--$9$ & not $8$--$10$ / $10$ / $8$--$9$ \\
Batch size & $8$ & $8$ \\
Student NFE & $1$ (unguided) & $1$ (unguided) \\
Teacher NFE / guidance & $10$ / $10$ & $10$ / $10$ \\
Teacher setting & $x_1$ look-ahead & $x_1$ look-ahead \\
Accept if $\Delta r_{\mathrm{Sei}}$ & $\ge 0.002$ & $\ge 0.002$ \\
Query time $t$ & $0$ & $0$ \\
Advantage & $\operatorname{clip}(10\Delta r,0,1)$ & $\operatorname{clip}(10\Delta r,0,1)$ \\
Rollout EMA coefficient & $1$ (frozen) & $1$ (frozen) \\
Seed / eval seed & $123$ / $7$ & $123$ / $7$ \\
Learning rate & $2.5\times 10^{-5}$ & $2.5\times 10^{-5}$ \\
Grad clip & $0.5$ & $0.5$ \\
$\beta$ & $0.8$ & --- \\
$\gamma$ / $\lambda$ & --- & $0.01$ / $0.01$ \\
Target extrapolation & $1.5$ & $\kappa=1.5$ \\
Max steps (selection) & $100$ ($90$) & $100$ ($90$) \\
\bottomrule
\end{tabular}}
\end{center}
\end{table}

Both DNA losses act on continuous RMF $x_1$ endpoint coordinates, using the convention of noise at $t=0$ and data at $t=1$. This differs from the image convention in~\eqref{eq:xt}. The learner calls the RMF endpoint predictor at $(t,s)=(0,1)$; its data target is the guided endpoint, optionally extrapolated from the frozen rollout endpoint as specified in~\eqref{eq:dna-loss}. Thus this adaptation does not use $F=x_t-tV$ at image time $t=0$, which would have zero residual.

Figure~\ref{fig:dna}b--c evaluates both methods every $10$ steps through step $100$. Under the EMA convention of Algorithm~\ref{alg:mfa}, $\rho=1$ freezes the rollout parameters, while the student remains trainable. Evaluation samples from the updated learner; the frozen rollout model is used only to construct training pairs and anchor predictions.

\section{Additional text-to-image results}
\label{app:image-extra}

\begin{figure}[htbp]
\centering
\includegraphics[width=\linewidth]{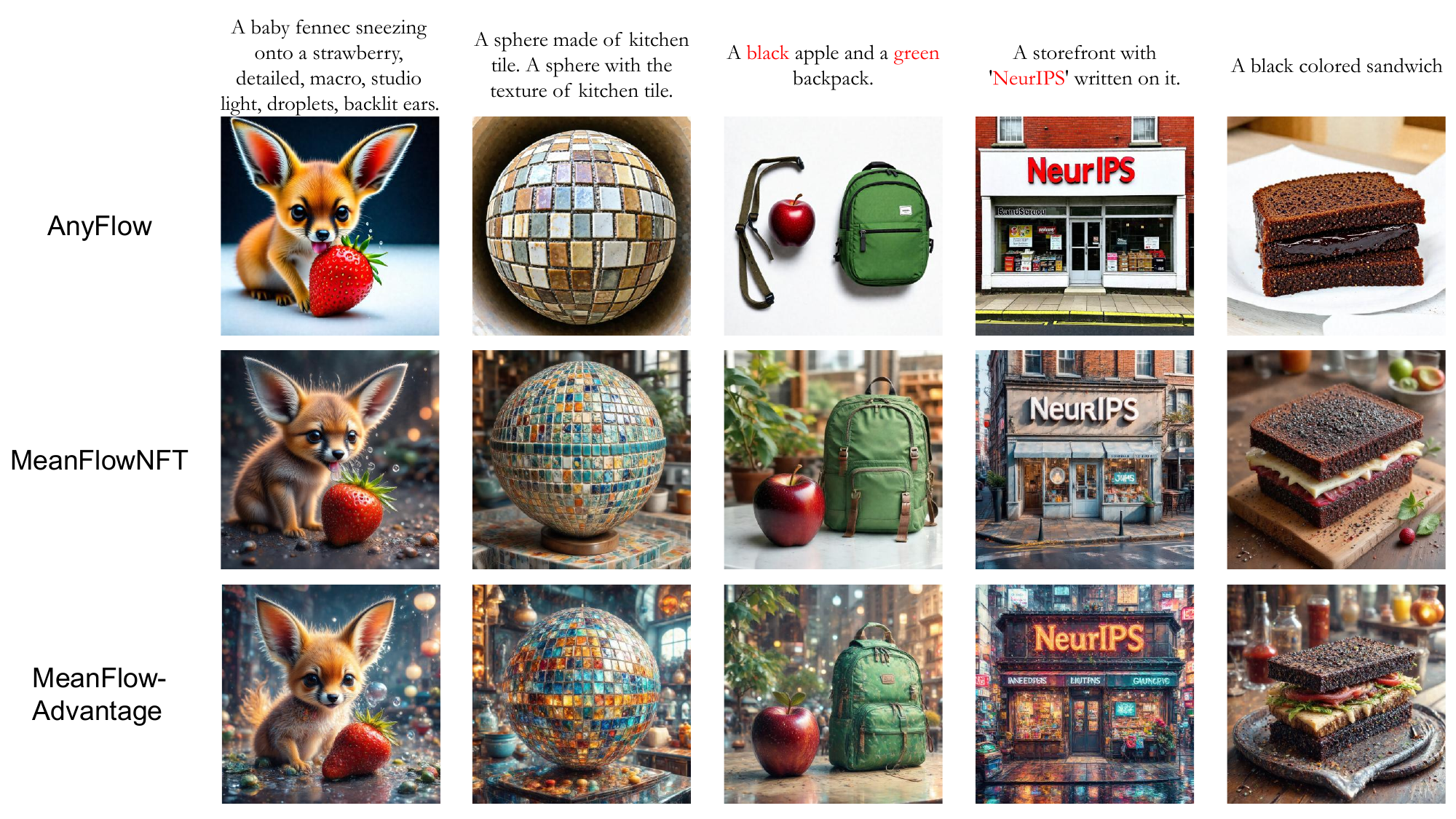}
\caption{More Qualitative Comparison Cases. The prompts are taken from GenEval, OCR and DrawBench respectively,where we compare the corresponding MeanFlowNFT model with our model.}
\label{fig:QualCom-si}
\end{figure}

\subsection{Matched image-model ablations}
\label{app:image-ablation-table}
The complete matched ablation endpoints used by the discussion in Section~\ref{sec:ablation} are reported below. These experiments provide robust directional insights that fully reinforce our primary findings.

For completed runs, training statistics are averages over the final 50 optimization steps and the held-out score is the final fixed-interval online evaluation (step 1950).
The two unstable positive-only controls were stopped at step 300 after the predefined collapse screen; their final online evaluation is at step 250.

\begin{table}[t]
\caption{Matched image-model ablations. ``Reward'' is the mean training reward
over the final 50 recorded steps of each run; ``Held-out Sum'' is the fixed
DrawBench online aggregate used for the ablation study. The early-stopped
positive-only runs are not equal-budget endpoints.}
\label{tab:mfa-ablation}
\centering
\resizebox{\linewidth}{!}{%
\begin{tabular}{lrrrrr}
\toprule
Variant & Train step & Reward$\uparrow$ & Held-out Sum$\uparrow$ & Grad. norm$\downarrow$ & s/step$\downarrow$ \\
\midrule
Shared induced $V$ (control) & 2000 & 1.5493 & 8.9841 & 0.97 & 73.3 \\
Direct $u$ (no induced correction) & 2000 & 1.5071 & 8.8915 & 2.10 & 64.8 \\
$s=t$ diffusion-only & 2000 & 0.8734 & 2.9590 & 445.79 & 65.4 \\
Separate derivative & 2000 & 1.5486 & 9.0382 & 1.00 & 89.9 \\
No adaptive $w$ & 2000 & 1.4959 & 8.6952 & 11.97 & 87.4 \\
$\gamma=5$ & 2000 & 1.5424 & 8.9936 & 0.84 & 72.9 \\
Positive $A$ only & 300$^\dagger$ & 0.7956 & 2.3275 & 1.30 & 73.4 \\
$A\equiv1$ & 300$^\dagger$ & 0.8260 & 3.7154 & 24.17 & 73.6 \\
$\gamma(A)=1-A$ & 2000 & 1.5482 & 9.0378 & 0.93 & 73.1 \\
\bottomrule
\end{tabular}}
\vspace{2pt}

{\footnotesize $^\dagger$Early-stopped after collapse screening; held-out evaluation is from step 250. The ablation aggregate uses the online normalized scoring convention and is not on the same numerical scale as the raw PickScore column in Table~\ref{tab:baselines}.}
\end{table}

Figure~\ref{fig:ablation-diagnostics-placeholder} reports the optimization diagnostics for these matched runs. The secondary comparisons below complement the three core findings in Section~\ref{sec:ablation}.

\begin{figure}[t]
\centering
% % Replace this box with the finalized optimization diagnostic figure, e.g.:
\includegraphics[width=\linewidth]{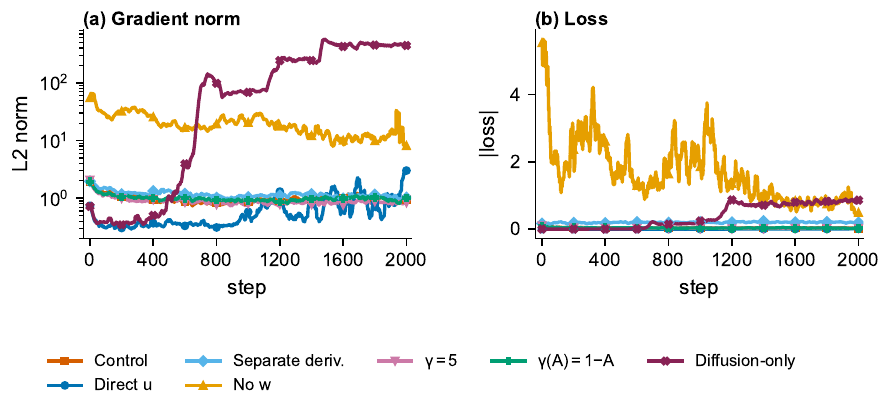}
% \fbox{\parbox[c][1.70in][c]{0.94\linewidth}{\centering
% \textbf{Ablation figure placeholder: optimization diagnostics}\\[3pt]
% Data/old/reference losses, gradient norm, advantage magnitude/variance,\\
% positive-negative fractions, and $\min(A+\gamma+\lambda)$.}}
\caption{Optimization stability across MeanFlowAdvantage ablations, highlighting gradient norms, loss dynamics, and collapse behavior.}
\label{fig:ablation-diagnostics-placeholder}
\end{figure}

\textbf{Sharing the derivative is the efficient default.}
Computing a separate derivative correction for learner, rollout, and reference
models reaches essentially the same reward regime as the shared correction,
but increases wall-clock time per step from 73.3 to 89.9 seconds. Its slightly
higher single-seed held-out aggregate is too small to establish a quality
advantage. This matches the algebra in \eqref{eq:shared-cd}: sharing the
correction makes the anchor residuals exact differences of the deployed
average-velocity networks while avoiding two additional derivative estimates.

\textbf{The anchor coefficient is not the main source of the gain.}
Increasing the fixed anchor coefficient from $\gamma=1.1$ to $\gamma=5$ leaves
the final reward and held-out aggregate nearly unchanged, while the dynamic
$\gamma(A)=1-A$ variant is also stable and reaches a comparable endpoint.
These single-seed differences are too small to rank the stable variants, but
they indicate that the principal effect is not a narrowly tuned value of
$\gamma$. We therefore retain $\gamma=1.1$ as the conservative default and
view $\gamma(A)=1-A$ as a useful follow-up rather than a claimed improvement.

\textbf{Positive local curvature is necessary but not sufficient.}
All stable configurations maintain $A+\gamma+\lambda>0$, as required by the
pointwise quadratic analysis in Section~\ref{sec:obj}. However, the
$s=t$ diffusion-only run also retains a positive minimum coefficient while its
data, reference, and gradient terms diverge late in training. The ablation
therefore clarifies the scope of the theory: positive curvature prevents the
per-sample free-prediction quadratic from becoming concave, but it cannot by
itself guarantee neural-network optimization stability or cross-interval flow
consistency.

\section{DNA promoters: on-policy RL and reward-graded distillation}
\label{sec:expts-dna}

The DNA experiments ask two questions. Does MFA work as on-policy RL for a generator whose state space is a manifold rather than a latent grid? And must the data targets come from the model's own rollouts at all? Nothing in Section~\ref{sec:analysis} requires the latter: per input, the loss~\eqref{eq:mfa-u} needs only a target, a scalar weight, and an anchor.

\textbf{Shared setup.} We use a Riemannian MeanFlow (RMF) generator~\citep{stark2024dirichlet,woo2026rmf} pretrained locally on $1024$-bp FANTOM5 promoters and conditioned on the regulatory signal; the reward is the negative Sei profile MSE~\citep{chen2022sei}. RMF places noise at time $0$ and data at $1$, so intervals run from $t$ to $s\ge t$ and the average velocity lies in the tangent space of the sphere at $x_t$. We train on the training chromosomes, select checkpoints on $256$ validation sequences (chr10) by Sei MSE, and report $512$ test sequences (chr8--9), evaluated one-step with a fixed prior seed and hard one-hot decoding. The pretrained generator reaches $0.0714$ Sei MSE and $0.928$ $6$-mer correlation.

\textbf{On-policy RL.} MFA is used exactly as in Section~\ref{sec:method}, with no teacher. Per signal, the frozen rollout model draws $K=8$ one-step sequences, Sei scores them, and~\eqref{eq:adv} standardizes the rewards within the group, so a sequence below its group average gets $A<0$. Each sequence is re-noised and the loss~\eqref{eq:mfa-u} is applied in the tangent space at $x_t$, with $\bar u$ and the shared correction from the rollout network as in RMF pretraining; intervals follow~\eqref{eq:three-mode}. MeanFlowNFT sees the same rollouts, advantages, intervals, and correction, differing only in the loss. Tables~\ref{tab:hp-dna-onpolicy} and~\ref{tab:dna-onpolicy-seeds} give the configuration and per-seed numbers.

Two things stand out (Table~\ref{tab:dna}, three seeds). First, teacher-free RL works on this manifold generator: MFA cuts the Sei MSE from $0.0714$ to $0.0482$ on average, and its best seed reaches $0.0437$, below the distilled models below. Second, MFA and NFT are statistically indistinguishable on the training reward: all three paired intervals contain zero. This is what Proposition~\ref{prop:nft} predicts, since a rollout model that tracks the learner leaves only a gain difference between the two objectives. They differ in the statistic that is \emph{not} optimized: MFA holds a $6$-mer correlation of $0.953\pm0.001$ against $0.937\pm0.009$, in all three seeds. We report this as a consistent trend over three seeds, not a tested claim.

\textbf{Reward-graded distillation.} We now replace the on-policy target by a teacher, the setting of scientific design, where a strong but slow sampler exists and the goal is a one-step generator that keeps its gains. From the same noise, a frozen one-step rollout gives an endpoint $b$ and a frozen ten-step Sei-guided teacher gives $g$. Their reward gap decides both whether a pair is used and how strongly:
\begin{equation}
\Delta r=r_{\mathrm{Sei}}(g)-r_{\mathrm{Sei}}(b),\qquad
\text{keep the pair if }\Delta r\ge\tau,\qquad
A=\operatorname{clip}(\eta\,\Delta r,0,1),
\label{eq:dna-two-stage}
\end{equation}
with $\tau=0.002$ and $\eta=10$. For each kept pair, the target is extrapolated past the teacher, and the rollout endpoint serves as both anchors:
\begin{equation}
y_\kappa=b+\kappa(g-b),\qquad
\ell_{\mathrm{DNA}}=A\|y_\theta-y_\kappa\|_2^2+(\gamma+\lambda)\|y_\theta-b\|_2^2,
\label{eq:dna-loss}
\end{equation}
with $\kappa=1.5$ and $\gamma=\lambda=0.01$. MeanFlowNFT receives the same pairs, advantages, extrapolation, learning rate, and validation protocol through its own loss, with $\beta=0.8$ and reference weight $k=0.01$.

\textbf{Why graded advantages separate the two losses.} As in~\eqref{eq:F-star}, each pair has a closed-form optimum. For MFA and for NFT (via the expansion in Appendix~\ref{app:nft-proof}),
\begin{equation}
y^\star_{\mathrm{MFA}}-b=\frac{\kappa A}{A+\gamma+\lambda}\,(g-b),\qquad
y^\star_{\mathrm{NFT}}-b=\frac{\kappa A}{\beta+k}\,(g-b).
\label{eq:dna-opt}
\end{equation}
MFA separates \emph{where} a pair points from \emph{how much} it counts: since $\gamma+\lambda=0.02$, every accepted pair targets nearly the full extrapolated endpoint and $A$ mainly sets the curvature $A+\gamma+\lambda$, that is, how hard the pair pulls. In NFT the advantage rescales the target itself, so pairs with a small but real improvement barely move while pairs with $A\approx1$ overshoot to $1.85(g-b)$. With $A\equiv1$ both reduce to a fixed extrapolation (Appendix~\ref{app:dna-extra}).

\begin{table}[htbp]
\caption{One-step promoter generation on the test chromosomes ($n=512$), from the same RMF initialization and the same protocol within each block. Sei MSE is the training reward; the $6$-mer correlation is never optimized. On-policy rows are mean $\pm$ s.d.\ over three training seeds, with the best seed in brackets; distillation rows are the single reported run. $\Delta$ is the paired NFT$-$MFA Sei MSE with a $95\%$ sequence bootstrap interval, so a positive $\Delta$ favors MFA.}
\label{tab:dna}
\centering
\small
\begin{tabular}{lccc}
\toprule
Method & Sei MSE$\downarrow$ & $6$-mer corr.$\uparrow$ & $\Delta$ [$95\%$ CI] \\
\midrule
Pretrained RMF & $0.0714$ & $0.928$ & --- \\
\midrule
\multicolumn{4}{l}{\emph{On-policy RL, no teacher (3 seeds)}}\\
MeanFlowNFT & $\mathbf{0.0471}\pm0.0048$ & $0.937\pm0.009$ & \multirow{2}{*}{$-0.0012$ $[-0.0051,0.0026]$} \\
MeanFlowAdvantage & $0.0482\pm0.0051$ $[0.0437]$ & $\mathbf{0.953}\pm0.001$ & \\
\midrule
\multicolumn{4}{l}{\emph{Reward-graded distillation from a $10$-step Sei-guided teacher}}\\
MeanFlowNFT & $0.0570$ & $\mathbf{0.955}$ & \multirow{2}{*}{$\mathbf{0.0103}$ $[0.0064,0.0143]$} \\
MeanFlowAdvantage & $\mathbf{0.0467}$ & $0.953$ & \\
\bottomrule
\end{tabular}
\end{table}

\textbf{Distillation results.} Both methods select step $90$ on validation. MFA reaches a Sei MSE of $0.0467$, $18\%$ lower than NFT on identical pairs and advantages, and here the gap is real: the paired difference is $0.0103$ with a $95\%$ interval of $[0.0064,0.0143]$ (Table~\ref{tab:dna}), and Figure~\ref{fig:dna} shows that it persists throughout training. Because that interval only resamples sequences, we repeated the whole comparison with four training seeds; Table~\ref{tab:dna-distill-seeds} shows the gap survives training randomness as well. The contrast with the on-policy block is the point of this section. With graded teacher improvements the two objectives use the reward differently, as~\eqref{eq:dna-opt} makes explicit, and with $A\equiv1$ they become indistinguishable again (Table~\ref{tab:dna-matched}). Distillation also loses the calibrated negative direction of on-policy RL: a teacher sample worse than the student does not say which direction is better. Unfiltered signed variants performed poorly (Appendix~\ref{app:dna-extra}), and we leave negative directions for distillation to future work.

\subsection{Additional promoter results}
\label{app:dna-fig-anchor}

\begin{figure}[htbp]
\centering
\includegraphics[width=\linewidth]{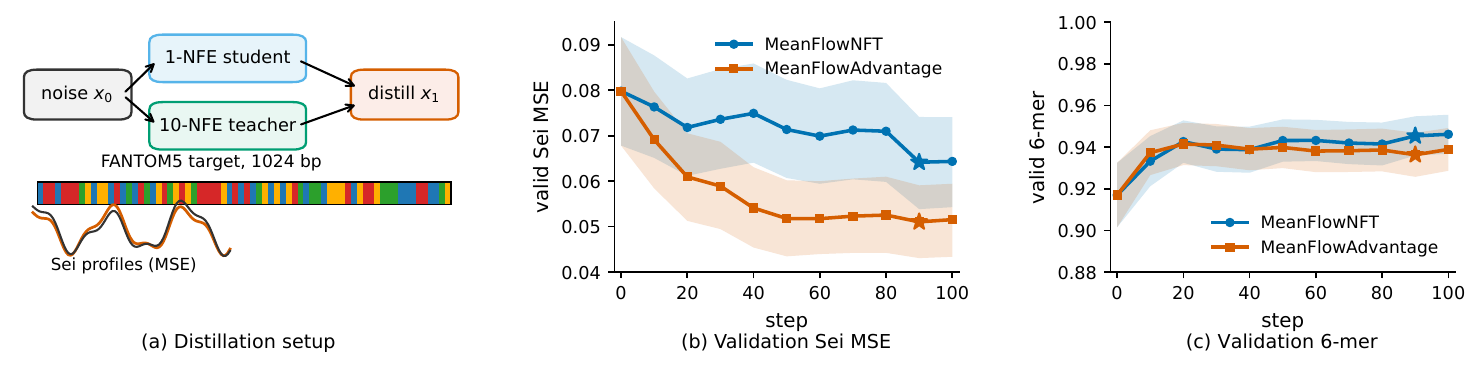}
\caption{Reward-graded MeanFlow distillation. (a) A frozen rollout and a Sei-guided teacher define the reward gap $\Delta r$ and the advantage $A$. (b)--(c) Validation curves ($n=256$); bands are sequence-bootstrap $95\%$ intervals, and stars mark the checkpoints selected by Sei MSE.}
\label{fig:dna}
\end{figure}

\label{app:dna-extra}

\textbf{Training-seed repeats of the distillation gap.}
The interval quoted in Table~\ref{tab:dna} resamples test sequences from
one training run, so it says nothing about how much the gap moves when
training is repeated. We therefore reran the whole comparison, both
losses and the full protocol, with four training seeds. Everything else
is held fixed: the same pretrained generator, the same teacher, the pair
filter $\tau=0.002$, $\kappa=1.5$, ten teacher steps, $\eta=10$, and the
same validation selection. Table~\ref{tab:dna-distill-seeds} lists the
result. MFA wins on every seed, and the spread across seeds is small
next to the gap itself: the paired difference is $0.0115$ with a
seed-level standard deviation of $0.0011$, and a bootstrap that
resamples seeds and sequences together gives $[0.0091,0.0138]$. The
reduction quoted in the main text is the smallest of the four
($18.1\%$); the four seeds average $19.7\%$. MFA is also the steadier of
the two, varying by $0.0005$ across seeds against $0.0010$ for NFT.
We did not sweep the remaining knobs of~\eqref{eq:dna-two-stage}, so the
filter threshold, the extrapolation $\kappa$, the teacher budget, and
the reward scale $\eta$ are fixed throughout at the values above; the
$A\equiv1$ rows of Table~\ref{tab:dna-matched} vary $\kappa$ and the
learning rate only.

\begin{table}[htbp]
\caption{Reward-graded distillation repeated over four training seeds.
Checkpoints are selected on valid ($n{=}256$) and reported on test
($n{=}512$, eval seed $7$). $\Delta$ is the paired NFT$-$MFA Sei MSE,
with a $95\%$ bootstrap over sequences; positive favors MFA. Seed $123$
is the run reported in Table~\ref{tab:dna}.}
\label{tab:dna-distill-seeds}
\begin{center}
{\small
\begin{tabular}{lcccccc}
\toprule
 & \multicolumn{2}{c}{MeanFlowNFT} & \multicolumn{2}{c}{MeanFlowAdvantage} & & \\
\cmidrule(lr){2-3}\cmidrule(lr){4-5}
Seed & Sei MSE$\downarrow$ & $6$-mer$\uparrow$ & Sei MSE$\downarrow$ & $6$-mer$\uparrow$
 & $\Delta$ [$95\%$ CI] & Red. \\
\midrule
$123$ & $0.0570$ & $0.955$ & $\mathbf{0.0467}$ & $0.953$ & $0.0103$ $[0.0064,0.0143]$ & $18.1\%$ \\
$42$ & $0.0585$ & $0.957$ & $\mathbf{0.0467}$ & $0.953$ & $0.0118$ $[0.0075,0.0162]$ & $20.2\%$ \\
$456$ & $0.0586$ & $0.958$ & $\mathbf{0.0477}$ & $0.947$ & $0.0110$ $[0.0068,0.0152]$ & $18.7\%$ \\
$789$ & $0.0594$ & $0.958$ & $\mathbf{0.0465}$ & $0.949$ & $0.0129$ $[0.0084,0.0173]$ & $21.7\%$ \\
\midrule
Mean & $0.0584$ & $0.957$ & $\mathbf{0.0469}$ & $0.951$ & $0.0115$ & $19.7\%$ \\
Std. & $0.0010$ & $0.001$ & $\mathbf{0.0005}$ & $0.003$ & $0.0011$ & $1.6\%$ \\
\bottomrule
\end{tabular}}
\end{center}
\vspace{2pt}
{\footnotesize Resampling seeds and sequences jointly ($4000$ draws)
gives a $95\%$ interval of $[0.0091,0.0138]$ for the mean paired
difference, so the gap does not depend on the seed that was reported.}
\end{table}

\textbf{Constant-advantage special case.}
Setting $A\equiv1$ discards the reward-gap magnitude after filtering. Because the rollout and reference anchors both equal $b$, the two DNA objectives then share an extrapolative regression form:
\begin{equation}
y_\theta-b=\frac{\kappa}{1+\gamma+\lambda}(g-b)
\quad\text{(MFA)},\qquad
y_\theta-b=\frac{e}{\beta+k}(g-b)
\quad\text{(MeanFlowNFT)},
\label{eq:dna-degenerate}
\end{equation}
where $e$ is the NFT target-extrapolation multiplier and $k$ its reference weight. Table~\ref{tab:dna-matched} shows comparable results when learning rate and nominal extrapolation are matched. The earlier individually tuned $A\equiv1$ configurations are retained in Tables~\ref{tab:dna-valid}--\ref{tab:dna-seeds} as ablations.

\begin{table}[htbp]
\caption{$A\equiv1$ special case under matched learning rate and nominal extrapolation. Checkpoints are selected on valid ($n=256$) and evaluated on test ($n=512$). $\Delta$ is paired NFT$-$MFA Sei MSE with a $95\%$ bootstrap interval.}
\label{tab:dna-matched}
\begin{center}
\resizebox{\linewidth}{!}{
\begin{tabular}{ccccc}
\toprule
lr & $\kappa$ & NFT MSE$\downarrow$ & MFA MSE$\downarrow$ & $\Delta$ [$95\%$ CI] \\
\midrule
$10^{-5}$ & $1.25$ & $0.0550$ & $0.0545$ & $0.0005$ $[-0.0007,0.0018]$ \\
$10^{-5}$ & $1.5$ & $0.0490$ & $0.0490$ & $0.0000$ $[-0.0010,0.0010]$ \\
$2.5{\times}10^{-5}$ & $1.25$ & $0.0525$ & $0.0518$ & $0.0007$ $[-0.0013,0.0028]$ \\
$2.5{\times}10^{-5}$ & $1.5$ & $0.0474$ & $0.0482$ & $-0.0009$ $[-0.0025,0.0008]$ \\
\bottomrule
\end{tabular}}
\end{center}
\end{table}

\textbf{Negative-direction ablation.}
In on-policy RL, negative advantages suppress sampled actions relative to the rollout policy. Here, a teacher-worse endpoint does not define a calibrated opposite target: unfiltered signed variants reach only $0.0744$--$0.0755$ best valid Sei MSE, versus $0.0510$ for filtered graded-advantage MFA. We therefore retain the filter and leave distillation-specific negative directions for future work. Figure~\ref{fig:dna-si} collects the $A\equiv1$ diagnostics: pretraining stability, paired valid differences, and evaluation-size sensitivity. Figure~\ref{fig:dna-si-onpolicy} shows training-pair teacher/student Sei MSE for the same special case. Table~\ref{tab:dna-seeds} varies the training seed for the earlier tuned $A\equiv1$ configurations while fixing the RMF initialization and eval seed $7$.

\begin{table}[htbp]
\caption{Individually tuned $A\equiv1$ special case on valid
$1$-NFE FANTOM5 ($n=512$, chr10). Online selection used $n=256$.}
\label{tab:dna-valid}
\begin{center}
\begin{tabular}{lccc}
\toprule
Method & step & Sei MSE$\downarrow$ & $6$-mer$\uparrow$ \\
\midrule
Pretrained RMF & $0$ & $0.0754$ & $0.924$ \\
MeanFlowNFT & $70$ & $0.0553$ & $0.951$ \\
MeanFlowAdvantage & $55$ & $\mathbf{0.0501}$ & $0.950$ \\
\bottomrule
\end{tabular}
\end{center}
\end{table}

\begin{table}[htbp]
\caption{Paired valid Sei $\Delta$MSE (NFT $-$ MFA) for the
individually tuned $A\equiv1$ special case, $n{=}512$,
$2000$ bootstrap draws.
The interval is the $95\%$ percentile CI.}
\label{tab:dna-paired}
\begin{center}
\begin{tabular}{lcc}
\toprule
step & $\Delta$MSE & $95\%$ CI \\
\midrule
$0$ & $0.000$ & $[0.000,0.000]$ \\
$10$ & $0.0103$ & $[0.0065,0.0141]$ \\
$50$ & $0.0065$ & $[0.0031,0.0101]$ \\
$60$ & $0.0058$ & $[0.0025,0.0092]$ \\
$70$ & $0.0052$ & $[0.0022,0.0083]$ \\
$100$ & $0.0081$ & $[0.0044,0.0116]$ \\
\bottomrule
\end{tabular}
\end{center}
\end{table}

\begin{table}[htbp]
\caption{FANTOM5 training-seed ablation for the individually tuned $A\equiv1$ special case.
All runs start from the same RMF \texttt{ckpts/0120} snapshot.
Checkpoints are selected on valid ($n{=}256$, eval seed $7$) and
 reported on test ($n{=}512$, eval seed $7$).
Only the training seed changes.}
\label{tab:dna-seeds}
\begin{center}
{\small
\begin{tabular}{lcccccc}
\toprule
 & \multicolumn{3}{c}{MeanFlowNFT} & \multicolumn{3}{c}{MeanFlowAdvantage} \\
\cmidrule(lr){2-4} \cmidrule(lr){5-7}
Seed & step & Sei MSE$\downarrow$ & $6$-mer$\uparrow$
 & step & Sei MSE$\downarrow$ & $6$-mer$\uparrow$ \\
\midrule
$123^{\dagger}$ & $70$ & $0.0542$ & $0.956$
 & $55$ & $\mathbf{0.0486}$ & $0.958$ \\
$42$ & $60$ & $0.0550$ & $0.955$
 & $55$ & $\mathbf{0.0492}$ & $0.955$ \\
$456$ & $100$ & $0.0550$ & $0.961$
 & $45$ & $\mathbf{0.0496}$ & $0.951$ \\
$789$ & $90$ & $0.0551$ & $0.955$
 & $35$ & $\mathbf{0.0479}$ & $0.953$ \\
\midrule
Mean & --- & $0.0548$ & $0.957$
 & --- & $\mathbf{0.0488}$ & $0.954$ \\
Std. & --- & $0.0004$ & $0.003$
 & --- & $0.0007$ & $0.003$ \\
\bottomrule
\end{tabular}}
\end{center}
\vspace{2pt}
{\footnotesize $^{\dagger}$Original tuned $A\equiv1$ run.
Std.\ is the sample standard deviation over the four seeds.}
\end{table}

\begin{figure}[htbp]
\centering
\includegraphics[width=\linewidth]{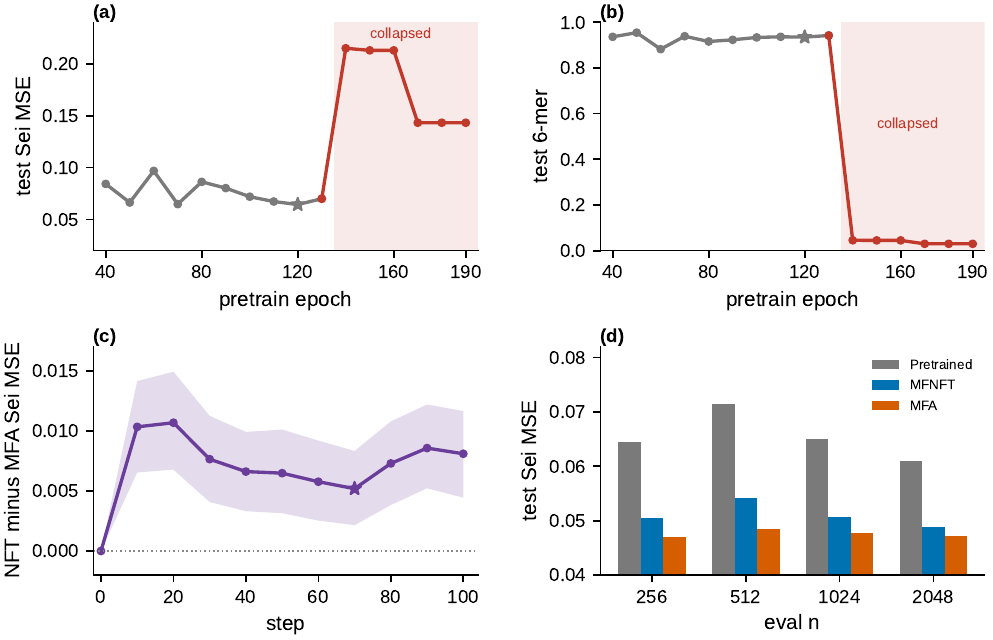}
\caption{Promoter analyses that are not in Figure~\ref{fig:dna}.
(a)--(b)~Test $1$-NFE Sei MSE and $6$-mer versus pretrain epoch
($n{=}256$).
The star is the RL init (\texttt{ckpts/0120}); the shaded region is
the collapsed tail.
(c)~Paired valid $\Delta$MSE (NFT $-$ MFA) with $95\%$
sequence-level bootstrap CIs ($n{=}512$).
The star is the on-grid NFT pick (step $70$).
(d)~Test Sei MSE at the tuned $A\equiv1$ checkpoints versus
CAGE-sorted eval $n\in\{256,512,1024,2048\}$.}
\label{fig:dna-si}
\end{figure}

\begin{figure}[htbp]
\centering
\includegraphics[width=0.62\linewidth]{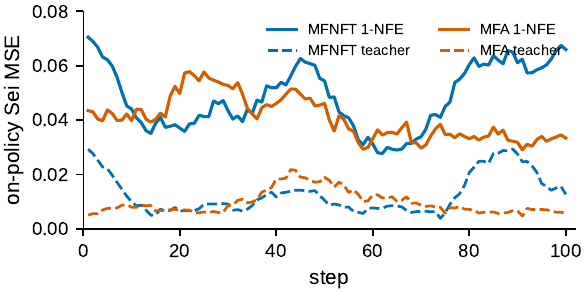}
\caption{Training-pair Sei MSE of accepted distillation pairs
(length-$15$ rolling mean).
Solid: $1$-NFE student. Dashed: $10$-NFE teacher.
These scores are computed for the accept filter
($\Delta r_{\mathrm{Sei}}\ge 0.002$); Table~\ref{tab:dna} instead
uses hard-decoded held-out sequences scored by Sei.}
\label{fig:dna-si-onpolicy}
\end{figure}

\end{document}